\documentclass{article}
\usepackage{microtype}
\usepackage{graphicx}
\usepackage{subfigure}
\usepackage{booktabs}
\usepackage{hyperref}
\usepackage{color}
\usepackage{xspace} 
\usepackage[accepted]{icml2021}
\usepackage{dsfont}
\usepackage{amsthm}
\usepackage{multirow}
\usepackage{amsmath}
\usepackage[ruled, vlined, linesnumbered]{algorithm2e}

\let\oldnl\nl% Store \nl in \oldnl
\newcommand{\nonl}{\renewcommand{\nl}{\let\nl\oldnl}}% Remove line number for one line

\newcommand{\tok}{\textit{top}\xspace}
\newcommand{\arm}{\textit{arm}\xspace}
\newtheorem{proposition}{Proposition}

\newcommand{\topk}{\ensuremath{\textnormal{\textit{top}-{\it k}}}\xspace}
\newcommand{\phat}{\widehat{\mu}}

\usepackage{enumitem}

\newtheorem{lemma}{Lemma}

\newtheorem{theorem}{Theorem}

\newtheorem{property}{Property}
\newcommand{\nop}[1]{}

\def \revise {\textcolor{blue}}
\input{0-commands}

\icmltitlerunning{Optimal Streaming Algorithms for  Multi-Armed Bandits}
\begin{document}
\onecolumn
\icmltitle{Optimal Streaming Algorithms for  Multi-Armed Bandits}
% \icmlsetsymbol{equal}{*}
\begin{icmlauthorlist}
\icmlauthor{Tianyuan Jin}{to1}
\icmlauthor{Keke Huang}{to2}
\icmlauthor{Jing Tang}{goo}
\icmlauthor{Xiaokui Xiao}{to3}
\end{icmlauthorlist}
\icmlaffiliation{to1}{ National University of Singapore, Tianyuan1044@gmail.com} 
\icmlaffiliation{to2}{ National University of Singapore, xkxiao@nus.edu.sg} 
\icmlaffiliation{goo}{The Hong Kong University of Science and Technology, Guangzhou, China,  jingtang@ust.hk} 
\icmlaffiliation{to3}{The University of British Columbia,  hkk992@gmail.com}
\icmlkeywords{Machine Learning, ICML}
\vskip 0.3in

\printAffiliationsAndNotice
\onecolumn

\revise{The version here includes a correction made to the published version: in Lemma 1, pages 5-7 and in Lemma 2, pages 18-21. The modified text is highlighted in blue. We sincerely appreciate Yucheng He for identifying the flaws in the earlier version of this paper and for the valuable discussions.}

\begin{abstract}

This paper studies two variants of the best arm identification (BAI) problem under the streaming model, where we have a stream of $n$ arms with reward distributions supported on $[0,1]$ with unknown means. The arms in the stream are arriving one by one, and the algorithm cannot access an arm unless it is stored in a limited size memory. 

We first study the streaming \eps-$top$-$k$ arms identification problem, which asks for $k$ arms whose reward means are lower than that of the $k$-th best arm by at most $\eps$ with probability at least $1-\delta$. For general $\eps \in (0,1)$, the existing solution for this problem assumes $k = 1$ and achieves the optimal sample complexity $O(\frac{n}{\eps^2} \log \frac{1}{\delta})$ using $O(\log^*(n))$ \footnote{$\log^*(n)$ equals the number of times that we need to apply the logarithm function on $n$ before the results is no more than 1.} memory and a single pass of the stream. We propose an algorithm that works for any $k$ and achieves the optimal sample complexity $O(\frac{n}{\eps^2} \log\frac{k}{\delta})$ using a single-arm memory and a single pass of the stream. 

Second, we study the streaming BAI problem, where the objective is to identify the arm with the maximum reward mean with at least $1-\delta$ probability, using a single-arm memory  and as few passes of the input stream as possible.  We present a single-arm-memory algorithm that achieves a near instance-dependent optimal sample complexity within $O(\log \Delta_2^{-1})$ passes, where $\Delta_2$ is the gap between the mean of the best arm and that of the second best arm.

% backup
% This paper studies two variants of the best arm identification (BAI) problem under the streaming model, where we have a stream of $n$ arms with reward distributions supported on $[0,1]$ with unknown means. The arms in the stream are arriving one by one, and the algorithm cannot access an arm unless it is stored in a limited size memory. 

% We first study the streaming \eps-\topk arms identification (\KAI) problem, which asks for $k$ arms whose reward means are lower than that of the $k$-th best arm by at most \eps~with at least $1-\delta$ probability. We propose a {\it single-pass} streaming algorithm that achieves the optimal sample complexity  $O(\frac{n}{\eps^2} \log\frac{k}{\delta})$ using a {\it single-arm memory}. 
% Second, we study the streaming BAI problem, i.e.,~\KAI with $\eps=0$ and $k=1$, that aims to identify the best arm with at least $1-\delta$ probability. We present an algorithm that achieves a near-optimal instance-dependent sample complexity of $O\big(\sum_{i=2}^{n}\frac{1}{\Delta_i^2} \log \big(\frac{1}{\delta}\log\frac{1}{\Delta_i} \big) \big)$ using a single-arm memory and $O(\log \Delta_2^{-1})$ passes, where $\Delta_i$ is the difference between the expected rewards of the best and the $i$-th best arms.
\end{abstract}

\section{Introduction}

Best arm identification (BAI) is a classic decision problem with numerous applications such as medical trials~\cite{thompson1933likelihood}, online advertisement~\cite{bertsimas2007learning}, and crowdsourcing~\cite{zhou2014optimal}. It typically considers a bandit with a set of arms, each of which has a reward distribution with an unknown mean. The objective is to identify the best arm with the maximum reward mean. 

Due to applications with massive data, the BAI problem has been recently studied under the streaming model in the literature~\cite{assadi2020exploration,falahatgaroptimal20,maiti2020streaming}, where only a limited size of memory is available for storing arms. In addition, BAI under the streaming model also avoids a large amount of time/money on switching alternatives and thus finds numerous applications. For example, in recruitment, employers aim to select the most qualified employee among all candidates with high probability. For this purpose, they could query each candidate with sufficient number of questions to acquire an accurate evaluation with confidence. The more questions they ask, the more confidence they have on the candidate's evaluation. Once the interview ends, usually, the candidate will not be asked for further evaluation.
%However, those candidates can only be interviewed one by one and cannot be queried back and forth frequently, since switching alternatives may delay the recruitment process and degrade the interview experience; 
In manufacturing, switching alternatives might require reassembling the production line, which could incur excessive costs. %\keke{we would better add some references here.}
%the number of arms allowed in memory is limited. In addition to memory considerations, streaming model also matches real applications, as described below.

%For some applications,  there is typically a cost for switching  alternatives. For example, in job interview, switching alternatives and querying back and forth between the applicants  may delay the process and  annoy the applicants. In manufacturing, switching alternatives might require reset the production line.   

%For example, in enterprise recruitment~\cite{falahatgaroptimal20}, you can query the applicant in an interview. Typically once you have finished the interview, you can not access to the applicant anymore. In enterprise recruitment, the goal is to find the most talent applicant when you have finished all the interviews.  

%The streaming model, pioneered by~\cite{assadi2020exploration}, mainly captures above application (with no switching  between alternatives) and the memory requirement.  In this paper, we studied the streaming Best Arms Identification problem.
%We formally define Streaming MAB as follows:  
%\cite{LiauSPY18,ChaudhuriK19}

Motivated by above observations, in this paper, we study two problems, \ie \textit{streaming \eps-\topk arms identification (\KAI)} and \textit{streaming BAI}. 

\spara{(Problem 1) Streaming \KAI} 
In streaming \KAI, we have a stream of $n$ arms, such that each $\arm_i$ is associated with an unknown reward distribution supported on $[0,1]$ with an unknown mean $\mu_{\arm_i}$. The arms in the stream are arriving one by one, and we can pull an arm only when it is stored in the memory. %\footnote{In~\cite{assadi2020exploration}, the algorithm can query the arm if it is stored in the memory, or the arm is arriving in the stream, while in ours, we can only access to the arm stored in the memory.}.  
Given parameters $\eps, \delta \in (0, 1)$, the task is to identify $k$ arms whose reward means are lower than that of the $k$-th best arm by at most $\eps$ with probability at least $1-\delta$. The ultimate goal in this paper is to minimize the sample complexity using a single-arm memory and a single pass over the stream.

\spara{(Problem 2) Streaming BAI}
In streaming BAI, the task is to identify the {\it optimal} arm with the largest mean with probability at least $1-\delta$, using a single-arm memory, assuming that there exists a unique optimal arm. Streaming BAI can be regarded as a special case of $\eps$-KAI with $\eps=0$ and $k=1$, 
Again, we aim to minimize the sample complexity using a single-arm memory and as few passes of the input stream as possible.
%To ensure the uniqueness of the solution, we assume the problem has  a unique optimal arm:  there exists an $\arm^*$ such that $\mu_{\arm^*}\geq \max\{\mu_{\arm_i}:\arm_i\neq \arm^*\} $.   In this problem, we allow to go through the stream several times, i.e., the arm arriving with order $\arm_1,\cdots,\arm_n,\arm_1,\cdots,\arm_n,\cdots$.  If the algorithm needs to go through the $n$ arms $L$ times, then it is a $L$ passes algorithm. We consider two metrics of sample complexity in BAI problem. In both metrics, only a single-arm-memory is allowed.    

% {\bf Metric 1: Instance Dependent Sample Complexity.}
% Let $\Delta_i$ be the difference of means between the best arm and the $i$-th largest arm.
% In this sample complexity metric,  we want to achieve near optimal instance dependent sample complexity $O\left(\sum_{i=2}^n \Delta_i^{-2} \cdot \log \frac{ \log \Delta_i^{-1}}{\delta}\right)$  whil passes as possible. We have designed a single-arm-memory algorithm that achieves $O\left(\sum_{i=2}^n \Delta_i^{-2} \cdot \log \frac{ \log \Delta_i^{-1}}{\delta}\right)$ sample complexity within $O(\log \Delta_2^{-1})$ passes. Now, the main challenge is the lower bound part. We think $\Omega(\log \Delta_2^{-1})$ is also the lower bound on pass complexity.

% {\bf Metric 2 : Asymptotic Sample Complexity.}  
% In this metric, we want to achieve the asymptotic optimal sample complexity  while using as small passes as possible. 
% I have designed a single-arm-memory algorithm that is asymptotic optimal within $O(1)$ passes.

 \begin{table*}[t]
%\vskip 0.15in
\centering
	\renewcommand{\arraystretch}{1.0}% for the vertical padding
\label{tbl:limited}
% \begin{sc}
\begin{tabular}{clcccc}
\toprule
   Problem      &  Algorithm & Requirement & Memory & \#passes  \\\midrule
{\multirow{5}{*}{\BAI}}
& \citet{assadi2020exploration}  & No  & $O(\log^*(n))$ & $1$ \\
& \citet{assadi2020exploration}  & $\eps\leq \Delta_2$ & $2$ & $1$  \\
%&\citet{maiti2020streaming} & Random-order arrival & $2$ & $1$ \\
& \citet{falahatgaroptimal20} & Random-order arrival & $1$ & $1$ \\
& \textbf{This thesis}   & No & $1$ & $1$ \\
\midrule
{\multirow{2}{*}{\KAI}}
& \citet{assadi2020exploration}  &  $\eps\leq\Delta_{k+1}$ & $O(k)$ & $1$  \\
& \textbf{This thesis}  & No & $1$ & $1$\\
\midrule
BAI & \textbf{This thesis} & No & $1$ & $O(\log \Delta_2^{-1})$\\
\bottomrule 
\end{tabular}
% \end{sc}
\caption{\qquad Comparisons of streaming algorithms for $\eps$-KAI with the optimal sample complexity of $O(\frac{n}{\varepsilon^2}\log \frac{k}{\delta})$ and BAI with sample complexity of $O\big(\sum_{i=2}^{n}\frac{1}{\Delta_i^2} \log \big(\frac{1}{\delta}\log\frac{1}{\Delta_i} \big) \big)$.}
% (for $i>k$, $\Delta_i$ denotes the difference between the means of the $k$-{th} and $i$-{th}  arms)
\vspace{-4mm}
\end{table*}

\subsection{State of the Art}
\label{state-art}

%Streaming \BAI (i.e., \eps-KAI with $k=1$) is studied in very recent work~\cite{assadi2020exploration,maiti2020streaming,falahatgaroptimal20}. Without any additional assumptions as we consider, the best known algorithm achieves the optimal sample complexity of $O(\frac{n}{\epsilon^2} \log \frac{1}{\delta})$ using $O(\log^*(n))$ memory~\cite{assadi2020exploration}. When it is limited to {\it single-arm memory} while still maintaining the optimal sample complexity, those algorithms assume either the arms in the stream are in a random order, or the gap between the means of the arms is larger than some constant, as summarized in Table~\ref{tbl:limited}. The assumptions make these algorithms quite restrictive. Detailed analyses are elaborated as follows.
\spara{Streaming \BAI} The \BAI problem under the streaming model is pioneered by \citet{assadi2020exploration}, for which a single-pass streaming algorithm is proposed that can achieve the optimal sample complexity of $O(\frac{n}{\varepsilon^2} \log \frac{1}{\delta})$ using $O(\log^*(n))$ memory. When $\eps\leq \Delta_2$, they further devise a single-pass algorithm using memory for pulling $2$ arms (\ie~the current arriving arm in the stream, and the candidate arm currently stored) while achieving the same sample complexity, where $\Delta_i$ is the difference between the expected rewards of $k$-th best and the $i$-th best arms for any $i>k$ and $k=1$ for \BAI. 
\citet{maiti2020streaming} reveal that if the arms arrive in a random order, the requirement of $\eps\leq \Delta_2$ can be discarded experimentally. However, if the arms arrive in some specific sequences, the correctness of the algorithm is not guaranteed. Moreover, both algorithms may revisit a candidate arm tested previously based on the sample outcomes of other arriving arms, which are often undesirable in practice. For example, during an interview process, an employer cannot repeatedly test a candidate based on the outcomes of other applicants, since a candidate is usually waiting at home for the final result after attending an interview. \citet{falahatgaroptimal20} propose an algorithm that tests each arm in a strictly first-in-first-out (FIFO) order but still assumes a random-order arrival of the arms.

\spara{Streaming $\eps$-KAI} %The $\varepsilon$-KAI problem is a standard generalization of \eps-BAI problem and is extensively studied in the literature~\cite{zhou2014optimal,agarwal2017learning,NIPS2019_8341,cao2015top,kalyanakrishnan2010efficient,kalyanakrishnan2012pac,chen2016pure}.  It is known that the optimal complexity for \eps-KAI is $\Theta(\frac{n}{\varepsilon^2} \log \frac{k}{\delta})$~\cite{kalyanakrishnan2012pac,cao2015top}. 
To the best of our knowledge, the work by \citet{assadi2020exploration} is the only one that studies the general streaming \KAI problem. Under the assumption that $\eps\leq \Delta_{k+1}$, \citet{assadi2020exploration} propose an algorithm that achieves the optimal sample complexity using $O(k)$ memory. Again, their algorithm suffers from two major deficiencies that it (i) does not test each arm in a FIFO order and (ii) requires $\Delta_{k+1}$ to be known in advance, which are unrealistic in many practical applications. 

 \vspace{-1mm}
\subsection{Our Contributions}
As our main result, we address the aforementioned shortcomings of existing algorithms for the general streaming \KAI problem and also study the streaming BAI problem which aims to identify the best arm strictly. The results are summarized in Table~\ref{tbl:limited}.
\eat{{\bf Streaming \eps-KAI.} %We first present our main result for \eps-KAI problem. We emphasize that in our problem setting, we only need to return the index of the arms. Hence, it is possible to design a single-arm-memory algorithm for \eps-KAI. 
 \citet{assadi2020exploration} first studies the streaming \eps-KAI problem. 
 ~For the case of $k = 1$, \citet{assadi2020exploration} presents two algorithms that achieve the optimal sample complexity $O(\frac{n}{\eps^2} \log \frac{1}{\delta})$ using a single pass of the stream. The first algorithm requires using $O(\log^*(n))$ memory, while the second algorithm uses a single-arm memory but assumes that $\eps < \Delta_2$.  Meanwhile, for general $k$, \citet{assadi2020exploration} proposes an algorithm that achieves the optimal sample complexity $\Theta(\frac{n}{\eps^2}\log \frac{k}{\delta})$ using a single pass and $O(k)$ memory, {\it assuming that $\eps < \Delta_{k+1}$}.} 

\spara{Streaming \KAI} We propose a single-pass algorithm for \KAI that achieves the optimal sample complexity using a {\it single-arm memory}, i.e., we pull an arm only at the time that it arrives and never revisit it after we pull other arms.\footnote{Note that we ignore the memory cost of storing the IDs of arms; otherwise, any algorithm for \KAI requires $\Omega(k)$ memory for recording the IDs of the arms to be returned.} Our solution significantly improves upon the algorithms by \citet{assadi2020exploration} in the way that (i) it does not rely on any assumption on \eps, and (ii) it requires only a single-arm memory for the general \KAI problem.

 \spara{Streaming BAI} We present an algorithm for streaming BAI that optimizes the sample complexity. Given any constant $\delta$, it achieves a near instance-dependent optimal sample complexity of $O\big(\sum_{i=2}^{n}\frac{1}{\Delta_i^2} \log \big(\frac{1}{\delta}\log\frac{1}{\Delta_i} \big) \big)$ using a single-arm memory and $O(\log \frac{1}{\Delta_{2}})$ passes in expectation.
 %\keke{Missing description of the second algorithm}
%There are two challenges concluded in their paper. 1: \cite{assadi2020exploration} needs to maintain a $k$ arms set $\cA$ in the memory.  When the arriving arm comes, the algorithm can not compare the arriving arm to every arm in $\cA$ (which will lead to $O(kn)$ sample complexity);
 % 2: there may be no gaps between any of two top-$k$ arms, thus it is hard to get any outcome when compare the top-$k$ arms with each other. The above differences shows that compared with $\eps$-BAI, the hardness of $\eps$-KAI problem raised significantly.

%\begin{table*}[t]
%\centering
%\begin{tabular}{ccccc}
%\hline
 %        &  Algorithm  & Memory Cost & Passes &  Query Complexity \\ \hline
%&\cite{jamieson2014lil} & $O(n)$ & $O(\log \Delta_2^{-1})$ & $O\left(\sum_{i=1}^n \Delta_i^{-2} \cdot \log \frac{\log \Delta_i^{-1}}{\delta}\right)$   \\
 %& This paper (Algorithm~\ref{alg:instance-dependent})  & $O(1)$  & $O(\log \Delta_2^{-1})$ & $O\left(\sum_{i=1}^n \Delta_i^{-2} \cdot \log \frac{\log \Delta_i^{-1}}{\delta}\right)$  \\
%\hline
%\end{tabular}
%\caption{Summary of algorithms for $\eps$-BAI.}
%\vspace{-3mm}
%\label{tbl:exact}
%\end{table*}

\section{Single-Arm Memory Algorithm for \BAI}\label{sec:main}
\vspace{-1mm}

In this section, we present our solution for the streaming \BAI problem, \ie~\KAI with $k=1$. We then extend our solution to address the \KAI problem for the general case of $k$ in Section 3.

\vspace{-1mm}
\subsection{High Level Overview}\label{sec:overview}
 
Let $\arm^o$ be the selected arm, and $\arm_i$ be the $i$-th arm in the stream where $i\in \{1,2,\ldots, n\}$. Let $\mu_{\arm}$ and $\phat_{\arm}$ be \arm's true mean and estimated mean respectively, and $\arm^\ast$ be the best arm. In the first step, we initialize $\arm^o$ with $\arm_1$. When $\arm_i$ arrives, we compare $\arm_i$ with $\arm^o$ and decide whether $\arm_i$ should be the new $\arm^o$. In particular, our algorithm mainly consists of the following two operations.
\begin{enumerate}[topsep=0mm,itemsep=0mm,leftmargin=*]
    \item \textbf{Sampling.} We pull each arrived arm $\Theta(\frac{1}{\eps^2} \log \frac{1}{\delta})$ times to estimate its true mean. This number of pull is sufficient to ensure that $\phat_{\arm^\ast}$ approaches $\mu_{\arm^\ast}$ within $O(\eps)$ with high probability\footnote{We omit {\it with high probability} in the following for expression simplicity.}, \ie $|\phat_{\arm^\ast}-\mu_{\arm^\ast}| \le O(\eps)$. 
    \item \textbf{Comparison.} We replace $\arm^o$ with $\arm_i$ if $\phat_{\arm_i} \ge \phat_{\arm^o} + \alpha$, where $\alpha=\Theta(\eps)$ is a random variable following a predefined distribution.\eat{is the core to guarantee the optimal sample complexity} 
\end{enumerate}
% For operation one, we pull each arrived arm $\Theta(\frac{1}{\eps^2} \log \frac{1}{\delta})$ times to estimate its true mean. Let $\mu_{\arm}$ and $\phat_{\arm}$ be \arm's true mean and estimated mean respectively, and $\arm^\ast$ be the best arm. This number of pull is sufficient to ensure that $\phat_{\arm^\ast}$ approaches $\mu_{\arm^\ast}$ within $O(\eps)$ with high probability\footnote{We omit {\it with high probability} in the following for expression simplicity.}, \ie $|\phat_{\arm^\ast}-\mu_{\arm^\ast}| \le O(\eps)$. For operation two, we replace $\arm^o$ with $\arm_i$ if $\phat_{\arm_i} \ge \phat_{\arm^o} + \alpha$ holds, where $\alpha=\Theta(\eps)$ is the core to guarantee optimal sample complexity (Details are in Appendix~\ref{appendix-1}). Meanwhile, our algorithm maintains the following property.
Our algorithm maintains the following property.
\begin{property}\label{ppt:arm-o}
    Whenever we update $\arm^o$, we always ensure that $|\phat_{\arm^o}-\mu_{\arm^o}| \le O(\eps)$.
\end{property}

\begin{algorithm}[!t]
\caption{Streaming \BAI\label{alg:main}}
\KwIn {$\varepsilon$, $\delta$, and a stream of $n$ arms.}
\KwOut {The index of an arm.} 
initialize $\arm^o \leftarrow \arm_1$ and $j\leftarrow 1$\;
pull $\arm^o$ $s_1$ times\;
%pull the first arm $\arm_1$ $s_1$ times\;
%initialize $\arm^o \leftarrow \arm_1$ and $j\leftarrow 1$\; \keke{First initialize, then pull?}
\ForEach{arriving $\arm_i$ ($i > 1$)}
{
    % sample $\alpha$ from the following distribution:
    % $\Pr(\alpha = x) \gets
    % \begin{cases}
    % \frac{1}{\log j + 1}, & \textrm{if $x = \frac{\varepsilon}{4}$} \\
    % % 1 - \frac{1}{\log j + 1}, & \textrm{if $x = \frac{\varepsilon}{2}$} \label{line:apha}
    % \end{cases}$ \\
    % $\Pr(\alpha=\frac{\varepsilon}{4})=1/(\log j+1)$ and $\Pr(\alpha=\frac{\varepsilon}{2})=1-1/(\log j+1)$\;
    set $\alpha$ as $\frac{\varepsilon}{4}$ with probability $\frac{1}{\log j + 1}$, and as $\frac{\varepsilon}{2}$ with other probability $1 - \frac{1}{\log j + 1}$\;\label{line:apha}
    $\ell \leftarrow 1$\;
    \While{true}
    {
        pull $\arm_i$ for $s_{\ell}-s_{\ell-1}$ times\;
        \If{$\phat_{\arm_i} \geq \phat_{\arm^o} + \alpha$ and $s_{\ell}>\tau_j$}  
        {
            $\arm^o\leftarrow \arm_i$\; \label{line:armchange}
            $j\leftarrow 1$\;
            \textbf{break}\;  
        }
        \ElseIf{$\phat_{\arm_i} < \phat_{\arm^o} + \alpha$}
        {
            $j\leftarrow j + 1$\; 
            \textbf{break}\; 
        }
        \Else
        {
            $\ell\leftarrow \ell+1$;
        }
    }
}
\Return the index of $\arm^o$;
\end{algorithm}

Based on the above two operations and Property~\ref{ppt:arm-o}, we could prove that $|\mu_{\arm^o(T)}-\mu_{\arm^\ast}| \le O(\eps)$ holds where $\arm^o(T)$ is the final returned arm. The basic idea is as follows. Operation one ensures that $|\phat_{\arm^\ast}-\mu_{\arm^\ast}| \le O(\eps)$. Operation two guarantees that $\phat_{\arm^o(T)} \ge \phat_{\arm^\ast}+\alpha$ if $\arm^o(T)$ is not $\arm^\ast$. In the meantime, $|\phat_{\arm^o(T)}-\mu_{\arm^o(T)}| \le O(\eps)$ holds according to Property~\ref{ppt:arm-o}. As a consequence, $|\mu_{\arm^o(T)}-\mu_{\arm^\ast}| \le O(\eps)$ is established.

As indicated above, the correctness of our algorithm lies in Property~\ref{ppt:arm-o}.  When each $\arm_i$ is pulled $\Theta(\frac{1}{\eps^2} \log \frac{cj^2}{\delta})$ times, we have $|\phat_{\arm_i}-\mu_{\arm_i}| \le O(\eps)$ with probability at least $1-\frac{\delta}{cj^2}$ according to Hoeffding bound, where $j$ is the number of arms that current $\arm^o$ beats and $c$ is a constant. Then by {\it union bound}, Property~\ref{ppt:arm-o} holds with probability at least $1-\sum^{\infty}_{j=1}\frac{\delta}{cj^2} \ge 1-\delta$.

% The essence for the correctness of our algorithm is to keep Property~\ref{ppt:arm-o}. To this end, we prove that when each $\arm^o$ is pulled $\Theta(\frac{1}{\eps^2} \log \frac{j^2}{\delta})$ times, we have $|\phat_{\arm^o}-\mu_{\arm^o}| \le O(\eps)$ with probability at least $1-\frac{\delta}{cj^2}$, where $j$ is the number of arms that $\arm^o$ beats and $c$ is a constant. Then by {\it union bound}, Property~\ref{ppt:arm-o} holds with probability at least $1-\sum^{\infty}_{j=1}\frac{\delta}{cj^2} \ge 1-\delta$. %\keke{we cannot claim $1-\delta$ probability here.}  
%$O(\frac{n}{\epsilon^2} \log \frac{1}{\delta})$

In what follows, we highlight how our algorithm achieves the optimal sample complexity when Property~\ref{ppt:arm-o} is maintained. As mentioned, each $\arm_i$ would be pulled $\Theta(\frac{1}{\eps^2} \log \frac{j^2}{\delta})$ times before it could replace current $\arm^o$. However, for arms with relatively small means, pulling such number of times is inefficient since we could identify and remove them with less pulls. 
\eat{To explain, suppose that $\mu_{\arm_i}$ is significantly smaller than $\mu_{\arm^o}$. By pulling $\arm_i$ $\Theta\big(\frac{1}{\varepsilon^2} \log \frac{1}{\delta} \big)$ times, $\phat_{\arm_i} < \phat_{\arm^o} + \alpha$ holds with high probability, as indicated by {\it Hoeffding Inequality}.} 
In this regard, we pull $\arm_i$ through multiple rounds. In the $\ell$-th round, it is pulled $\Theta(\frac{2^{\ell}}{\eps^2} \log \frac{1}{\delta})$ times. This round loop terminates immediately once either the number of pulls reaches $\Theta(\frac{1}{\eps^2} \log \frac{j^2}{\delta})$ or we are able to decide  $\arm_i$ is not  the best arm and then eliminate it.

Another aspect to optimize sample complexity lies in the design of $\alpha$. Actually, this part is the main {\it hardness} of our algorithm. One conventional method is to set a fixed value to $\alpha$. However, this would lead to suboptimal sample complexity. To explain, suppose we set $\alpha = \frac{\varepsilon}{2}$. If some $\arm_i$ has $\mu_{\arm_i}=\phat_{\arm^o}+\frac{\varepsilon}{2}$, it is inappropriate to bound the probability $\Pr(\phat_{\arm_i}\geq \phat_{\arm^o}+\frac{\varepsilon}{2})$ according to Hoeffding inequality. To fix this, a straightforward method is to bound the number of pulls of $\arm_i$ by $\Theta(\frac{1}{\eps^2} \log(\frac{j^2}{\delta}))$. However, when $j$ grows to $\Theta(n)$, this method would incur the total sample complexity of $O\big(\frac{n}{\varepsilon^2} \log \frac{n}{\delta}\big)$, which is suboptimial. To bypass this intractable issue, we leverage the \textit{power of randomization}. That is, we set $\alpha=\frac{\varepsilon}{4}$ with probability $\frac{1}{\log j+1}$ and $\alpha=\frac{\varepsilon}{2}$ with probability $1-\frac{1}{\log j+1}$. We elaborate the details later.

%The number of pulls of $\arm_i$ is upper bounded by $\tau_j$. However, this will not happen over times. The reason is that  when $\alpha=\frac{\varepsilon}{4}$, applying Hoeffding bound, with high probability $\arm_i$ will be the new $\arm^o$. This ensures that current $\arm^o$ can not eliminate too many arms with $\Omega(\tau_j)$ pulls.  

The proof of the optimal sample complexity is highly non-trivial and is also one of our {\it main technical} contributions. We refer readers to Appendix~\ref{appendix-1} for details.

\subsection{The Algorithm}\label{sec:algorithm}
We first introduce two parameters used in our algorithm.
\begin{align}
  &\{s_\ell\}_{\ell=1}^{\infty}\colon \quad s_\ell = \frac{16}{\varepsilon^2} \cdot \log\bigg(\frac{C}{\delta}\bigg)\cdot 2^{\ell}, \text{ and } s_{0}=0,
	%r_1 := 4, \qquad \qquad r_{\ell} = 2 \cdot r_{\ell-1}; 
\label{def:s_l}	\\ 
%	&\alpha = \alpha(\varepsilon,\delta) :=; \notag \\ \tag{intermediate parameter to define the number of samples per each level of the challenge} \\
	&\{\tau_j\}_{j=1}^{\infty}\colon \quad \tau_j :=  \frac{32}{\varepsilon^2} \cdot \log\bigg(\frac{C\cdot j^2}{\delta}\bigg),  \label{def:tau_j} %\notag
\end{align}
where $\ell$ indicates the $\ell$-th round and $C\geq 100$ is a universal constant.

Algorithm~\ref{alg:main} presents the pseudo-code of our algorithm. In the beginning, we initialize $\arm^o=\arm_1$ with the first arm $\arm_1$, and then pull $\arm^o$ $s_1$ times to obtain its estimated mean $\phat_{\arm^o}$. In what follows, for each arrived $\arm_i$ in the stream, we sample $\alpha$ from the distribution defined as
\begin{equation*}
    \Pr\big(\alpha = \frac{\varepsilon}{4}\big) = \frac{1}{\log j + 1}\ \text{and}\ \Pr\big(\alpha = \frac{\varepsilon}{2}\big) = 1-\frac{1}{\log j + 1},
\end{equation*}
% \begin{equation*}
%     \Pr(\alpha = x) =
%      \begin{cases}
%      \frac{1}{\log j + 1}, & \textrm{if $x = \frac{\varepsilon}{4}$}, \\
%      1 - \frac{1}{\log j + 1}, & \textrm{if $x = \frac{\varepsilon}{2},$}% \label{line:apha}
%      \end{cases}
% \end{equation*}
where $j$ is the number of arms beaten by $\arm^o$. Next, $\arm_i$ is compared with $\arm^o$ in multiple rounds. In the $\ell$-th round, $\arm_i$ will be pulled $s_{\ell}-s_{\ell-1}$ times to obtain its estimated mean $\phat_{\arm_i}$. If both conditions $\phat_{\arm_i} \ge \phat_{\arm^o}+\alpha$ and $s_{\ell} > \tau_j$ hold, (i) $\arm^o$ is replaced by $\arm_i$, (ii) $j$ is reset to $1$, and (iii) current round terminates. Otherwise, we would check whether the condition $\phat_{\arm_i} < \phat_{\arm^o} + \alpha$ meets. If it is true, $\arm_i$ will be removed immediately. Meanwhile, $j$ is increased by $1$ and the round ends. If none of the two events on round termination happen, we increase index $\ell$ by $1$ and then enter the next round. The above procedure is repeated for each arriving arm until all arms in the stream have been scrutinized. Eventually, we return the index of the final $\arm^o$. 
\eat{Notice that in this procedure, only a single-arm memory is consumed when pulling an arriving arm.}

\subsection{The Analysis}\label{sec:analysis}

We say that an arm is {\it \eps-best} arm if its mean is smaller than that of the best arm $\arm^\ast$ by at most $\varepsilon$, \ie $\mu_{\arm^\ast} - \mu_{\arm} \le \eps$. We formalize our main result for \eps-BAI problem as follows. 

\begin{theorem}\label{thm:main}
Given a stream of $n$ arms, approximation parameter $\varepsilon$ and confidence parameter $\delta$ in $(0, 1)$, Algorithm~\ref{alg:main} finds the $\varepsilon$-best arm with probability at least $1-\delta$ using expected $O(\frac{n}{\varepsilon^2}\log\frac{1}{\delta})$ pulls and a single-arm memory.   
\end{theorem}
%Due to the space constraint, we provide full proofs for the algorithm's correctness and a proof sketch for the optimal sample complexity. The detailed proofs for the sample complexity are in  

Let {\it best arm change} be the event that $\arm^o$ is replaced by another arm, and $\arm^o(t)$ denote the resulting arm after best arm change happens exactly $t$ times (Note that $\arm^o(1) = \arm_1$). We denote $\arm^o
(T)$ as the final returned $arm^o$. In what follows, we focus on the correctness proof of Algorithm~\ref{alg:main}, \ie $\mu_{\arm^o(T)}\geq \mu_{\arm^*}-\varepsilon$ holds with probability at least $1-\delta$. The proof consists of two parts. 
In the first part, we establish the relation between all $\arm^o$ and $\arm^\ast$ in Lemma~\ref{lem:main-1}. In the second part, we then complete the correctness proof based on the result of Lemma~\ref{lem:main-1}.

\begin{proposition}[Hoeffding Inequality]\label{prop:chernoff}
Let $X_1,\ldots,X_m$ be $m$ independent random variables with support in $[0,1]$. Define $X := \sum_{i=1}^{m} X_i$. Then, for  $x > 0$, 
\begin{align*}
	\Pr(X - \EE[X] > x) \leq 2 \cdot \exp\bigg(-\frac{2x^2}{m}\bigg). 
\end{align*}
\end{proposition}

\begin{lemma}
\label{lem:main-1}
%\begin{align}
 %   \Pr\bigg(\exists t\geq0:\phat_{\arm^o(t)}\geq \mu_{\arm^*}-\frac{5\varepsilon}{8},\text{ and }\mu_{\arm^o(t)}\leq \mu_{\arm^*}-\varepsilon\bigg)\leq \frac{3\delta}{4}.
%\end{align}
For any $\varepsilon, \delta\in(0,1)$, it holds in Algorithm~\ref{alg:main} that
\begin{align*}
    & \Pr\bigg(\bigcap_{t\geq1}\bigg\{\bigg\{\phat_{\arm^o(t)}< \mu_{\arm^*}-\frac{5\varepsilon}{8}\bigg\} \notag \\
    & \qquad \qquad \bigcup\bigg\{\mu_{\arm^o(t)}\geq \mu_{\arm^*}-\varepsilon\bigg\}\bigg\}\bigg)\geq 1-\frac{3\delta}{4}.
\end{align*}
\end{lemma}
\revise{
\begin{proof}[Proof of Lemma~\ref{lem:main-1}] 
%At $t=1$, Algorithm~\ref{alg:main} pulls $\arm^o(1)$ $s_1$ times. From Proposition~\ref{prop:chernoff}, for $r\in \mathbb{N}^+$, we have
%\begin{align}
 %     &\Pr\bigg( |\phat_{\arm^o(1)}-\mu_{\arm^o(1)}|\geq \frac{r\varepsilon}{8}\bigg)  \leq 2\exp \bigg( -\frac{s_1r^2\varepsilon^2}{32}\bigg) \notag \\
%      & =2\exp\bigg(-r^2 \log\bigg(\frac{C}{\delta} \bigg) \bigg)=\frac{2\delta^{r^2}}{C^{r^2}} \leq  \frac{2\delta}{C^r}.
%\end{align}
Let $Q_{t,p}$ be the $p$-th passed arm after $t$-th best arm change, and $s(p):=s_{\ell}$ such that $s_{\ell-1}< \tau_p\leq s_{\ell}$. For ease of analysis, we design a virtual sampling process for a better illustration. Notably, if $Q_{t,p}$ is pulled less than $\tau_p$ times when Algorithm~\ref{alg:main} ends, we pull $Q_{t,p}$ again to $s(p)$ times (a virtual process). Therefore, for all $p\geq 1$, $Q_{t,p}$ is pulled $s(p)$ times to obtain its estimated mean, denoted by $\phat^\prime_{Q_{t,p}}$. If ${\arm^o(t+1)}=Q_{t,p}$, then $\phat^\prime_{Q_{t,p}}= \phat_{Q_{t,p}}$ holds by definition. %Let $F^o(t)$ be the union of history till the $t$-th best arm change. 
%Then, conditioned on any $F^o(t)$, we have 
%\begin{align}
%\label{eq:-virtual}
% & \bigg\{ | \phat_{Q_{t,p}}-{\mu}_{Q_{t,p}} |\geq  \frac{r\varepsilon}{8}, {\arm^o(t+1)}=Q_{t,p} \bigg \} \notag \\ & \qquad \subseteq \bigg\{|\phat'_{Q_{t,p}}-{\mu}_{Q_{t,p}}|\geq  \frac{r\varepsilon}{8} \bigg\}.
%\end{align}
Let 
\begin{align*}
    & \cQ_d(T)=\bigg\{Q_{t,p}: \mu_{Q_{t,p}}\in \bigg(\mu_{\arm^*}-\frac{d\varepsilon}{8}, \notag \\
    & \quad \quad \mu_{\arm^*}-\frac{(d-1)\varepsilon}{8}\bigg],  p\geq 1, \text{ and } t\in [T] \bigg\},
\end{align*}
where $d$ is an integer and $d\geq 1$.  Fixed a $d\geq 9$, we consider the following cases. 
\\
\textbf{Case 1:} 
Let $G_{d}$ be the event that there exists $t$ and $p$ such that $\hat{\mu}_{\arm^o(t)}\leq \mu_{\arm^*}-\frac{d\eps}{4} -\frac{\eps}{2}$, $Q_{t,p}\in \cQ_{d}(T)$.
 Let $E_{d1}$ be the event that 1: $G_{d}$ occurs, and 2: for the first $t$ and $p$ such that $\hat{\mu}_{\arm^o(t)}\leq \mu_{\arm^*}-\frac{d\eps}{4} -\frac{\eps}{2}$ and $Q_{t,p}\in \cQ_{d}(T)$, it holds that $\arm^o(t+1)=Q_{t,p}$ and $\hat{\mu}_{\arm^{o}(t+1)}\in [\mu_{\arm^o(t+1)}-(d-8)\eps/8, \mu_{\arm^o(t+1)}+(d-8)\eps/8]$.  We note that according to the update rule of the algorithm, if for all $p\geq 1$, $| \phat'_{Q_{t,p}}-{\mu}_{Q_{t,p}} |\leq  \frac{(d-8)\varepsilon}{8}$, then $Q_{t,p}$ will be $\arm^o(t+1)$. Define $F_{t,p}$ be the union of history till $Q_{t,p}$ comes. 
Therefore,
\begin{align}
\label{eq:Ed1}
 & \min_{F_{t,p}: G_{d}  \text{ is true under } F_{t,p}}  \PP\bigg\{E_{d1} \ \bigg | \ F_{t,p} \bigg \}  \notag \\
  & \geq 1-\sum_{p\geq 1} \Pr\bigg(| \phat'_{Q_{t,p}}-{\mu}_{Q_{t,p}} |\leq  \frac{(d-8)\varepsilon}{8}  \bigg) \notag \\
  & \geq 1- \frac{2\delta}{p^2 C^{d-8}} \notag \\
  & \geq 1- \frac{4\delta}{C^{d-8}}.
\end{align}
\textbf{Case 2:} 
Let $G_{d}(t)$ be the event that $\hat{\mu}_{\arm^o(t)}> \mu_{\arm^*}-\frac{d\eps}{4} -\frac{\eps}{2}$.
Let $E_{d2}(t)$ be the event that 1: $G_{d}(t)$ is true, 2: for any $p\geq 1$ and $Q_{t,p}\in\cQ_{d}(T)$, $| \phat'_{Q_{t,p}}-{\mu}_{Q_{t,p}} |< \frac{(d-8)\varepsilon}{8}$. We have
\begin{align}
\small
\label{eq:Ed2}
& \qquad \PP\bigg\{ E_{d2}(t) \ \bigg | \ G_{d}(t)\bigg\} \notag \\
    & \geq 1-\PP\bigg\{ \exists p
\geq 1: Q_{t,p} \in \cQ_{d}(T), | \phat'_{Q_{t,p}}-{\mu}_{Q_{t,p}} |\geq  \frac{(d-8)\varepsilon}{8}
 \bigg \} \notag \\
  & \geq 1- \sum_{p\geq 1}\Pr\bigg(| \phat'_{Q_{t,p}}-{\mu}_{Q_{t,p}} |\geq  \frac{(d-8)\varepsilon}{8}  \bigg) \notag \\
  & \geq 1-\sum_{p\geq 1}\frac{2\delta}{p^2 C^{d-8}} \notag \\
  & \geq 1- \frac{4\delta}{C^{d-8}}.
\end{align}
When $G_{d}(t)$ and $E_{d2}(t)$ are true and $\arm^o(t+1)\in \cQ_{d}(T)$, we have
\begin{align*}
    | \phat_{\arm^o(t+1)}-{\mu}_{\arm^o(t+1)} |< \frac{(d-8)\varepsilon}{8}.
\end{align*}
Fixed a $d<9$, we consider the following cases. 
\\
\textbf{Case 3:} 
Let $E_{d3}$ be the event that 1: $G_{d}$ occurs, and 2: for the first $t$ and $p$ such that $\hat{\mu}_{\arm^o(t)}\leq \mu_{\arm^*}-\frac{d\eps}{4} -\frac{\eps}{2}$ and $Q_{t,p}\in \cQ_{d}(T)$, $\arm^o(t+1)=Q_{t,p}$ and $\hat{\mu}_{\arm^{o}(t+1)}\in [\mu_{\arm^o(t+1)}-d\eps/8, \mu_{\arm^o(t+1)}+d\eps/8]$. 
 We have
\begin{align}
\label{eq:Ed3}
  \min_{F_{t,p}: G_{d}  \text{ is true under } F_{t,p}}   \PP\bigg\{ E_{d3} \ \bigg | \ G_{d} \ \bigg \} & \geq  1-\sum_{p\geq 1}\Pr\bigg(| \phat'_{Q_{t,p}}-{\mu}_{Q_{t,p}} |\leq  \frac{d\varepsilon}{8} \bigg) \notag \\
  & \geq 1- \sum_{p\geq 1}\frac{2\delta}{p^2 C^{d}} \notag \\
  & \geq 1- \frac{4\delta}{C^{d}}.
\end{align}
\textbf{Case 4:}  
Let $E_{d4}(t)$ be the event that 1: $G_{d}(t)$ is true, 2: for any $p\geq 1$ and $Q_{t,p}\in\cQ_{d}(T)$, $| \phat'_{Q_{t,p}}-{\mu}_{Q_{t,p}} |< \frac{d\varepsilon}{8}$. Then, we have
\begin{align}
\label{eq:Ed4}
 \PP\bigg\{ E_{d4}(t) \ \bigg | \ G_{d}(t)\bigg\}%& \geq 1-\PP(E_{d2}^c(t)) \notag \\
     \geq &1- \PP\bigg\{ \exists p
\geq 1: Q_{t,p} \in \cQ_{d}(T), | \phat'_{Q_{t,p}}-{\mu}_{Q_{t,p}} |\geq  \frac{d\varepsilon}{8}, {\arm^o(t+1)}=Q_{t,p} 
 \bigg \} \notag \\
  & \geq 1- \sum_{p\geq 1}\Pr\bigg(| \phat'_{Q_{t,p}}-{\mu}_{Q_{t,p}} |\geq  \frac{d\varepsilon}{8}  \bigg) \notag \\
  & \geq 1-\sum_{p\geq 1}\frac{2\delta}{p^2 C^{d}} \notag \\
  & \geq 1- \frac{4\delta}{C^{d}}.
\end{align}
When $G_{d}(t)$ and $E_{d4}(t)$ are true and $\arm^o(t+1)\in \cQ_{d}(T)$, then 
\begin{align*}
    | \phat_{\arm^o(t+1)}-{\mu}_{\arm^o(t+1)} |< \frac{d\varepsilon}{8}.
\end{align*}
Define $I$ as a random variable such that
$I=1$ if the following conditions hold.
\begin{enumerate}
    \item For all $d\geq 9$, if $G_{d}$ is true,  so does $E_{d1}$.
    \item For all $d\geq 9$ and $t\in [T]$, if $G_{d}(t)$ is true, so does $E_{d2}(t)$.
    \item For all $d\in [1,8]$, if $G_{d}$ is true, so does $E_{d3}$.
    \item For all $d\in [1,8]$ and $t\in [T]$, if $G_{d}(t)$ is true, so does $E_{d4}(t)$.
\end{enumerate}
Otherwise, $I=0$.  We use $E^c$ for the opposite event of $E$. Similar to the chain rule, we sequentially apply the \eqref{eq:Ed1},\eqref{eq:Ed2},\eqref{eq:Ed3}, and \eqref{eq:Ed4} to compute the probability of $I=1$. 
We have 
\begin{align*}
    \PP(I=1)& \geq  \EE\Bigg[ \prod_{d\geq 9}\bigg( \ind\{G_{d}^c\}+ \ind\{G_{d}\}\PP(E_{d1}\mid G_{d}) \bigg) \prod_{d\geq 9}  \prod_{t\in [T]} \bigg(  \ind\{G_{d}^c(t)\}+\ind\{G_{d}(t)\}\PP(E_{d2}\mid G_{d}(t)) \bigg)\notag \\
    & \quad \times \prod_{d< 9}\bigg( \ind\{G_{d}^c\}+ \ind\{G_{d}\} \PP(E_{d3}\mid G_{d})\bigg) \prod_{d< 9}  \prod_{t\in [T]} \bigg(  \ind\{G_{d}^c(t)\}+\ind\{G_{d}(t)\}\PP(E_{d4}\mid G_{d}(t)) \bigg)  \Bigg]  \notag \\
   &  \geq \prod_{d\geq 9}  \min_{F_{t,p}: G_{d}  \text{ is true under } F_{t,p}}\PP(E_{d1}\mid F_{t,p}) \cdot  \EE\Bigg[\prod_{d\geq 9} \prod_{t\in [T]} \bigg(  \ind\{G_{d}^c(t)\}+\ind\{G_{d}(t)\}\PP(E_{d2}\mid G_{d}(t)) \bigg)\Bigg] \notag \\
    & \quad \times \prod_{d< 9}  \min_{F_{t,p}: G_{d}  \text{ is true under } F_{t,p}} \PP(E_{d3}\mid F_{t,p}) \cdot \EE \Bigg[\prod_{d< 9} \prod_{t\in [T]} \bigg(  \ind\{G_{d}^c(t)\}+\ind\{G_{d}(t)\}\PP(E_{d4}\mid G_{d}(t)) \bigg)\Bigg]
\end{align*}
Note that $\PP(E_{d2}\mid G_{d}(t))$ can be bounded by \eqref{eq:Ed2} and $\PP(E_{d4}\mid G_{d}(t))$ can be bounded by \eqref{eq:Ed4}. Assume
$I=1$. We now determine the number of occurrences of $G_{d}(t)$. For $I=1$,
 we have $\hat{\mu}_{\arm^{o}(T)}\leq U:= \mu_{\arm^*}+\epsilon/8$. Given that $\alpha\geq \frac{\varepsilon}{4}$, the update rule in Algorithm~\ref{alg:main} indicates $\phat_{\arm^o(t)}\geq \phat_{\arm^o(t-1)}+\frac{\varepsilon}{4}$.  This means that the number of $\arm^o(t)$ with $\hat{\mu}_{\arm^o(t)}>L:=\mu_{\arm^*}-\frac{d\eps}{4}-\frac{\eps}{2}$ is at most 
\begin{align}
  \frac{4(U-L)}{\eps}\leq d+3.
\end{align}
Therefore,
\begin{align}
    \Pr(I=1) & \geq \prod_{d=1}^{8} \bigg(1-\frac{4\delta}{C^{d}} \bigg)\prod_{d=1}^{8} \bigg(1-\frac{4(d+3)\delta}{C^{d}} \bigg)\prod_{d=9}^{\infty} \bigg(1-\frac{4\delta}{C^{d-8}}\bigg)\prod_{d=9}^{\infty} \bigg(1-\frac{4(d+3)\delta}{C^{d-8}} \bigg) \notag \\
    &\geq 1-\sum_{d=1}^{8}\frac{4(d+4)\delta}{C^{d}}-\sum_{d=9}^{\infty} \frac{4(d+4)\delta}{C^{d-8}} \notag\\
    &\geq 1-\frac{3\delta}{4}.
    \end{align}
where the last inequality is due to $C\geq 100$. 
Let $r_t \in \NN$ and satisfies
\begin{align}
\label{eq:clr}
 \mu_{\arm^o(t)}\in \bigg(\mu_{\arm^*}-\frac{r_t\eps}{8}, \mu_{\arm^*}-\frac{(r_t-1)\eps}{8} \bigg].
\end{align}
Based on the definition of $I$ and $\cQ_d(T)$,  we have 
\begin{align}
    \{I=1\} & \subseteq  \bigg\{\bigcap_{t\geq1} \bigg\{ \bigg\{|\phat_{\arm^o(t)}-\mu_{\arm^o(t)}|\leq\frac{(r_t-8)\varepsilon}{8}\bigg\}\notag \\
    &\qquad \qquad \bigcup   \bigg\{\mu_{\arm^o(t)}\geq \mu_{\arm^*}-\varepsilon\bigg\} \bigg\}\bigg\} \notag \\
     & \subseteq   \bigg\{\bigcap_{t\geq1} \bigg\{\bigg\{\phat_{\arm^o(t)}< \mu_{\arm^*}-\frac{5\varepsilon}{8}\bigg\} \notag \\
    & \qquad \qquad \bigcup\bigg\{\mu_{\arm^o(t)}\geq \mu_{\arm^*}-\varepsilon\bigg\}\bigg\} \bigg\},
\end{align}
where the second formula follows since if $|\phat_{\arm^o(t)}-\mu_{\arm^o(t)}|\leq\frac{(r_t-8)\varepsilon}{8}$, we have
\begin{align}
    &\phat_{\arm^o(t)}\leq \mu_{\arm^o(t)}+\frac{(r_t-8)\varepsilon}{8} \notag \\
    &\leq \mu_{\arm^*}-\frac{(r_t-1)\varepsilon}{8}+ \frac{(r_t-8)\varepsilon}{8} < \mu_{\arm^*}-\frac{5\varepsilon}{8}.
\end{align}
This completes the proof.
    \end{proof}
}

Based on Lemma~\ref{lem:main-1}, we are ready to accomplish the correctness of Algorithm~\ref{alg:main}.

\begin{proof}[Proof of Correctness of Algorithm~\ref{alg:main}]
 Since $\alpha\leq \frac{\varepsilon}{2}$,
 from Algorithm~\ref{alg:main},  there exists an $\arm^o(t)$ such that
\begin{align}
\label{eq:main-correct-1}
  \phat_{\arm^*}-\frac{\varepsilon}{2}\leq \phat_{\arm^o(t)}.
\end{align}
From Proposition~\ref{prop:chernoff}, we know
\begin{align}
\label{eq:the1-bound-muo}
\Pr\bigg(\phat_{\arm^*} \geq \mu_{\arm^*} -\frac{\varepsilon}{8} \bigg)&\geq 1-2\exp \bigg(-\frac{s_1\varepsilon^2}{8} \bigg)\notag \\
&\geq 1- \frac{\delta}{4}.
\end{align}
Since $\alpha\geq 0$, from Algorithm~\ref{alg:main}, we have 
\begin{align}
\label{eq:main-correct-22}
    \phat_{\arm^o(t)}\leq\phat_{\arm^o(T)}.
\end{align} 
Combining~\eqref{eq:main-correct-1}~\eqref{eq:the1-bound-muo}~\eqref{eq:main-correct-22}  together, we have
\begin{align}
\label{eq:proof-correct}
    \Pr\bigg(\phat_{\arm^o(T)}\geq \mu_{\arm^*}-\frac{5\varepsilon}{8}\bigg)\geq 1-\frac{\delta}{4}.
\end{align}
From Lemma~\ref{lem:main-1}, we obtain
\begin{align}
    &\Pr\bigg(\bigg\{\phat_{\arm^o(T)}< \mu_{\arm^*}-\frac{5\varepsilon}{8}\bigg\} \notag \\
    &\qquad \quad  \bigcup\bigg\{\mu_{\arm^o(T)}\geq \mu_{\arm^*}-\varepsilon\bigg\}\bigg) \geq 1-\frac{3\delta}{4}.
\end{align}
Let 
\begin{align*}
    &A=\bigg\{\phat_{\arm^o(T)}< \mu_{\arm^*}-\frac{5\varepsilon}{8}\bigg\},\\
    \text{and}\quad& B=\bigg\{\mu_{\arm^o(T)}\geq \mu_{\arm^*}-\varepsilon\bigg\}.
\end{align*}
Then from~\eqref{eq:proof-correct},  $\Pr(A)\leq \frac{\delta}{4}$. Therefore $\Pr(B)\geq \Pr(A \cup B)-\Pr(A)\geq 1-\delta$, which completes the proof.
% 
% 
% Define  event $E_1$: for all any $\arm^o(t)$ that satisfies $\mu_{\arm^o(t)}\leq \mu_{\arm^*}-\varepsilon$, $\phat_{\arm^o(t)}\leq \mu_{\arm^*}-\frac{5\varepsilon}{8}$. Define event $E_2=\{\phat_{\arm^o(T)}\geq \mu_{\arm^*}-\frac{5\varepsilon}{8}\}$. Assume event $E_1$ and $E_2$ holds.  Then we have $\mu_{\arm^o(t)}\geq \mu_{\arm^*}-\varepsilon$.
% From Lemma~\ref{lem:main-1}, $\Pr(E_1)\geq 1-\frac{3\delta}{4}$. 
% 
% 
% 
%\begin{align*}
 %   \Pr(\phat_{\arm^o(T)}\geq \mu_{\arm^*}-\frac{5\varepsilon}{8})\leq \Pr(\mu_{\arm^o(T)}\geq \mu_{\arm^*}-\varepsilon)
%\end{align*}
%which immediately implies $\Pr(\mu_{\arm^o(T)}\geq\mu_{\arm^*}-\varepsilon)\geq 1-\delta$. 
\end{proof}

Due to the space constraint, we provide a sketch of proof for the optimal sample complexity, and we refer interested readers to Appendix~\ref{appendix-1} for details. 

\begin{proof}[Proof of Sample Complexity of Algorithm~\ref{alg:main} (Sketch)] 
    The key idea is to bound the total expected number of pulls during the life cycle of each selected $\arm^o$, i.e., the period from $\arm^o$ replacing its predecessor to $\arm^o$ being replaced by its successor. Given current $\arm^o$, we divide the arriving arms during the life cycle of $\arm^o$ into two sets, i.e.,
    \begin{align*}
        &S_1:=\big\{\arm_i\colon \mu_{\arm_i} \leq \phat_{\arm^o}+\frac{3\varepsilon}{8}\big\},\\
        \text{and}\quad&S_2:=\big\{\arm_i\colon \mu_{\arm_i} > \phat_{\arm^o}+\frac{3\varepsilon}{8}\big\}.
    \end{align*}
    We show that (i) for each $\arm_i\in S_1$, the expected number of pulls of $\arm_i$ is $O\big(\frac{1}{\varepsilon^2}\log \frac{1}{\delta}\big)$, and (ii) the total expected number of pulls of all arms in $S_2$ is $O(\tau_j|S_2|)$ where $|S_2|$ is $O(\operatorname{polylog}(j))$ with high probability. Consequently, the total expected number of pulls for $j$ arriving arms during the life cycle of $\arm^o$ is $O\big(\frac{j}{\eps^2}\log\frac{1}{\delta} \big)$. In the following, we provide some intuitive analyses and the formal analysis is far more challenging and interested readers are referred to the appendix for more details.
    % Given the current $\arm^o(t)$, the crux is to bound the expected number of pulls for each arrived $\arm_i$ before the $(t+1)$-th best arm change.  We consider the following two cases.

    \spara{Case I} Consider $\arm_i\in S_1$ with $\mu_{\arm_i}\leq \phat_{\arm^o}+\frac{3\varepsilon}{8}$. By Hoeffding inequality, after $\Theta(\frac{1}{\eps^2}\log \frac{1}{\delta})$ pulls of $\arm_i$, $\hat{\mu}_{\arm_i}\leq {\mu}_{\arm_i}+\frac{\eps}{8}$ holds with high probability. Thus, $\hat{\mu}_{\arm_i}\leq \phat_{\arm^o}+\frac{\eps}{2}$. If $\alpha=\frac{\eps}{2}$, $\arm_i$ will be dropped. On the other hand, if $\alpha=\frac{\eps}{4}$, $\arm_i$ will be pulled at most $2\tau_j$ times. As a result, the expected number of pulls of $\arm_i$ is $O\big(\frac{1}{\eps^2}\log \frac{1}{\delta}\cdot (1-\frac{1}{\log j+1})+\frac{2\tau_j}{\log j+1}\big)=O\big(\frac{1}{\eps^2}\log \frac{1}{\delta}\big)$.
    
    % According to Hoeffding inequality, after $\Theta(\frac{1}{\eps^2}\log \frac{1}{\delta})$ pulls of $\arm_i$, $\hat{\mu}_{\arm_i}\leq {\mu}_{\arm_i}+\frac{\eps}{8}\leq \hat{\mu}_{\arm^o}+\frac{\eps}{2}$ holds with high probability. Therefore, for the case of $\alpha=\frac{\eps}{2}$, $\arm_i$ will be removed after $\Theta\big(\frac{1}{\eps^2}\log \big(\frac{1}{\delta}\big)\big)$ pulls with high probability. Notice that in our algorithm, $\arm_i$ is pulled at most $2\tau_j=O\big(\frac{1}{\eps^2}\log\big(\frac{j^2}{\delta}\big)\big)$ times. For $\alpha=\frac{\eps}{4}$ with probability $\frac{1}{\log j+1}$, the expected number of pulls is upper bounded by $\frac{2\tau_j}{\log j+1}= O\big(\frac{1}{\eps^2}\log\big(\frac{1}{\delta}\big)\big)$. By taking expectation on $\alpha$, the expected number of pulls of $\arm_i$ is $O\big(\frac{1}{\varepsilon^2}\log \big(\frac{1}{\delta}\big)\big)$.
    
    \spara{Case II} Consider $\arm_i\in S_2$ with $\mu_{\arm_i}> \phat_{\arm^o}+\frac{3\varepsilon}{8}$. Again, by Hoeffding inequality, after $\Theta(\frac{1}{\eps^2}\log \frac{1}{\delta})$ pulls of $\arm_i$, $\hat{\mu}_{\arm_i}\geq {\mu}_{\arm_i}-\frac{\eps}{8}$ holds with high probability. Thus, $\hat{\mu}_{\arm_i}>\phat_{\arm^o}+\frac{\varepsilon}{4}$. If $\alpha=\frac{\eps}{4}$, $\arm_i$ will replace $\arm^o$ and the life of $\arm^o$ ends. When $|S_2|=\Theta(\operatorname{polylog}(j))$, $\alpha=\frac{\eps}{4}$ will happen at least once with high probability.
    
    Putting it together, we have the total expected number of pulls for all the $j$ arms $O\big(j\cdot \frac{1}{\eps^2}\log \frac{1}{\delta}+\operatorname{polylog}(j)\cdot \tau_j\big)=O\big(\frac{j}{\eps^2}\log\frac{1}{\delta}\big)$, since $j=\Omega(\operatorname{polylog}(j))$.
    % In this case, the number of pulls for each $\arm_i$ is bounded by $2\tau_j$. Besides, if $\alpha=\frac{\eps}{4}$, $\arm_i$ will highly likely become the new $\arm^o$ since $\phat_{\arm_i} \ge \phat_{\arm^o} + \frac{\eps}{4}$ holds with high probability.
    % 
    % Intuitively, since $\Pr\big(\alpha=\frac{\eps}{4}\big)=\frac{1}{\log j+1}$, if there are $(\log j+1)^2$ arms with means larger than $\phat_{\arm^o}+\frac{3\varepsilon}{8}$, with high probability $\arm^o$ will be replaced by some $\arm_i$ in the stream and $j$ will be reset to $1$.  Now, we bound the number of pulls of the passed $j$  arms (before $j$ is reset).  According to the analysis of {\it Case I}, for $j-(\log j+1)^2$ arms with means smaller than $\phat_{\arm^o}+\frac{3\varepsilon}{8}$, the number of pulls is $O\big(\frac{j}{\eps^2}\log \big(\frac{1}{\delta}\big)\big)$. Besides, the number of pulls incurred by  $(\log j+1)^2$ arms is $O(2\tau_j\cdot (\log j+1)^2)=O\big(\frac{\log^2 j}{\varepsilon^2}\log \big(\frac{j^2}{\delta}\big)\big)$, which is dominated by $O\big(\frac{j}{\eps^2}\log \big(\frac{1}{\delta}\big)\big)$.  Hence, the total number of pulls is $O\big(\frac{j}{\eps^2}\log\big(\frac{1}{\delta} \big)\big)$. 
\end{proof}

\section{Single-Arm-Memory Algorithm for \KAI} \label{sec:eps-KAI}

In this section, we extend $\eps$-BAI into its general version, \ie \eps-KAI that aims to find the $\varepsilon$-\topk arms using a {\it single-arm memory}. That is, we aim to find $k$ arms such that each of which has the mean no smaller than $\mu_{\arm^\ast(k)}-\varepsilon$, where $\arm^\ast(k)$ is the $k$-th largest value in $\{\mu_{\arm_1},\ldots,\mu_{\arm_n}\}$.  

%\subsection{Main Theorem}
%In our \eps-KAI problem, we use the {\it Explore-k} metric and the goal is to select $k$ arms such that  with probability at least $1-\delta$, each selected  arms has mean greater than $\mu_{\arm^*(k)}-\varepsilon$, where $\mu_{\arm^*(k)}$ is $k^{t h}$ largest mean in  $\{\mu_{\arm_1},\ldots,\mu_{\arm_n}\}$.  

\subsection{High Level Overview}\label{sec:overviewk}

\eat{We say that $k$ arms are \eps-\topk arms if each of them has mean no smaller than $\mu_{\arm^*(k)}-\varepsilon$, where $\mu_{\arm^*(k)}$ is $k$-th largest value among $\{\mu_{\arm_1},\ldots,\mu_{\arm_n}\}$ for $n$ arms in the stream.} 

First, we maintain the first $k$ arms in a set\footnote{We store the IDs of these $k$ arms and ignore the memory cost.}, denoted as $\cA$. Let $\tok^o$ be the arm in $\cA$ with the minimum estimated mean. When the following $\arm_i$ arrives in the stream, we compare $\phat_{\arm_i}$  with $\phat_{\tok^o}$ to update $\tok^o$. Similarly, we perform the following two operations.
\begin{enumerate}[topsep=0mm,itemsep=0mm,leftmargin=*]
    \item \textbf{Sampling.}  We pull $\arm_i$ $\Theta(\frac{1}{\eps^2}\log\frac{k}{\delta})$ times to get $\phat_{\arm_i}$.
   % Let $\mu_{\arm}$ and $\phat_{\arm}$ be \arm's true mean and estimated mean respectively, and $\arm^\ast$ be the best arm. This number of pull is sufficient to ensure that $\phat_{\arm^\ast}$ approaches $\mu_{\arm^\ast}$ within $O(\eps)$ with high probability\footnote{We omit {\it with high probability} in the following for expression simplicity.}, \ie $|\phat_{\arm^\ast}-\mu_{\arm^\ast}| \le O(\eps)$. 
    \item \textbf{Comparison.} We replace $\tok^o$ with $\arm_i$ if $\hat{\mu}_{\arm_i}\geq \hat{\mu}_{\tok^o}+\alpha$ holds, where $\alpha = \Theta(\eps)$ follows the same setting in Algorithm~\ref{alg:main}.
\end{enumerate}
 In addition, our algorithm retains the following property.

\begin{property}\label{ppt:topk-arm}
    For each $\arm$ in $\cA$, $|\phat_{\arm}-\mu_{\arm}| \le O(\eps)$.
\end{property}
 %\keke{smaller than $O(\eps)$?}

%the estimated mean of arms in $\cA$ is not far from its true mean by $\Theta(\eps)$.

%\begin{enumerate}
%\item the estimated mean of  arms in $\cA$ is not far from its true mean by $\Theta(\eps)$.
%\end{enumerate}

Let $\tok^o(T)$ be the arm with the minimum estimated mean in the final returned set $\cA$ . Following the similar logic flow in section~\ref{sec:overview}, the two operations and Property~\ref{ppt:topk-arm} guarantee that $|\mu_{\tok^o(T)}-\mu_{\arm^*(k)}|\leq O(\eps)$ with high probability. The main idea is as follows. By applying union bound for the $k$ arms $\{\arm^\ast(1),\ldots, \arm^\ast(k)\}$, operation one ensures that for all $s\in[k]$, $| \hat{\mu}_{\arm^*(s)}-\mu_{\arm^*(s)}|\le O(\eps)$. As operation two indicates, $\phat_{\tok^o(T)}$ is the $k$-th largest estimated mean among those of arms in $\cA$, which means there are at most $k-1$ values in $\{\phat_{\arm^*(1)},\cdots, \phat_{\arm^*(k)}\}$ larger than $\phat_{\tok^o(T)}$ by $\Theta(\eps)$. Therefore, based on operation one and operation two, we have $\phat_{\tok^o(T)}-\phat_{\arm^*(k)} \le O(\eps)$. Meanwhile, $|\hat{\mu}_{\tok^o(T)}-\mu_{\tok^o(T)}| \le O(\eps)$ holds according to Property~\ref{ppt:topk-arm}. In consequence, we acquire $|{\mu}_{\tok^o(T)}-{\mu}_{\arm^*(k)}|\leq O(\eps)$ and the $k$ arms in $\cA$ are $\eps$-$\tok$-$k$ arms with high probability.

%the first operation ensures that $\min_{s\in[k]} \hat{\mu}_{\arm^*(s)}-\mu_{\arm^*(k)}\geq \Theta(\eps)$. From operation 2 and the fact that $\phat_{\tok^o(T)}$ is the $k$-th largest value stored in $\cA$, hence at most $k-1$ values in   $\{\phat_{\arm^*(1)},\cdots, \phat_{\arm^*(k)}\}$ could be  larger than  $\phat_{\tok^o(T)}$ by $\Theta(\eps)$.  In other words, the distance between the estimated mean of $\tok^o(T)$ and $\arm^*(k)$ is no larger than $\Theta(\eps)$. Meanwhile, $|\hat{\mu}_{\tok^o(T)}-\mu_{\tok^o(T)}|=O(\eps)$ holds according to {\it property 2}. In a conclusion, $|{\mu}_{\tok^o(T)}-{\mu}_{\arm^*(k)}|\leq O(\eps)$ is achieved. Finally, using the union bound for $k$ arms in returned $\cA$, with high probability $\cA$ is an $\eps$-$\tok$-$k$ arm set.
How to implement the two operations and Property~\ref{ppt:topk-arm} within optimal number of pulls remains the main challenge in $\eps$-KAI. However, techniques adopted in \eps-BAI could basically tackle this issue. Hence we omit the details here. 

%In \eps-KAI problem, the main challenge  is also  keeping above two operation and one property with the optimal number of pulls. The method we used to deal with the challenge is similar to the  method we used in \eps-BAI problem,  hence details are omitted here. 

\subsection{The Algorithm}\label{sec:algorithmk}

First we define the following two parameters of our algorithm.
\begin{align*}
 &\{s_\ell\}_{\ell=1}^{\infty}: \quad s_\ell = \frac{16}{\varepsilon^2} \cdot \log\bigg(\frac{C\cdot k}{\delta}\bigg)\cdot 2^{\ell}, \text{ and } s_0=0; \\
	&\{\tau_j\}_{j=1}^{\infty}: \quad \tau_j :=  \frac{32}{\varepsilon^2} \cdot \log\bigg(\frac{C\cdot k \cdot j^2}{\delta}\bigg), 
\end{align*}
where $C\geq 100$ is a universal constant. \eat{Notice that the number of returned arms $k$ is considered in definitions.}

Algorithm~\ref{alg:topk} presents the pseudo-code for \eps-KAI. As noticed, Algorithm~\ref{alg:topk} is tailored based on Algorithm~\ref{alg:main} to identify the $\eps$-$\tok$-$k$ arms. Specifically, to initialize $\cA$, we pull the first $k$ arrived arms $s_1$ times, and store them in $\cA$. $\tok^o$ denotes the arm with the minimum estimated mean in $\cA$. Then, we compare each arriving $\arm_i$ with $\tok^o$ through multiple rounds. This part is conducted similarly as the procedure of Algorithm~\ref{alg:main} in section~\ref{sec:algorithm}. Eventually, the indexes of $\eps$-$\tok$-$k$ arms in $\cA$ are returned.  

Our main result is formalized in the following theorem.
\begin{theorem}\label{thm:top-k}
Given a stream of $n$ arms, approximation parameter $\varepsilon$ and confidence parameter $\delta$ in $(0, 1)$, Algorithm~\ref{alg:topk} finds $\varepsilon$-top-k arms with probability at least $1-\delta$ using expected $O(\frac{n}{\varepsilon^2} \cdot \log{(\frac{k}{\delta})})$  pulls and a single-arm memory.
\end{theorem}
Theorem~\ref{thm:top-k} summarizes the main results of Algorithm~\ref{alg:topk}. Detailed proofs are in Appendix \ref{appendix-2}. 

Compared with \citet{assadi2020exploration}, our algorithm~\ref{alg:topk} is fundamentally different.  \citet{assadi2020exploration} require the assumption, e.g., $\Delta_{k+1}<\varepsilon$ and does not test each arm in FIFO order, which leads to $O(k)$ memory costs. In comparison, our algorithm does not make any explicit assumption and uses a single-arm memory.

{\bf Remark:} This paper adopts the wildly used Explore-$k$ metric \cite{kalyanakrishnan2010efficient} which asks for $k$ arms whose reward means are lower than that of the $k$-th best arm by at most $\eps$.  By using the similar technique  \cite{cao2015top,tianyuan2019efficient}, our algorithm can return the top-$k$ arms such that the mean of $i$-th returned arm is at most $\varepsilon$  lower than that of the $i$-th best arm.

%The first arms are added to $\cA$ iteratively. Denote the $\tok^o$ be the arm with the minimum estimated mean in $\cA$. Then, we initialize $j=1$. For each arriving $\arm_i$ in the stream, we sample a parameter $\alpha$ for it.   After that, we test $\arm_i$ on multiple levels. Each level we pull $\arm_i$ $s_{\ell}-s_{\ell-1}$ times. Denote $\phat_{\arm_i}$ for the estimated mean of $\arm_i$. In each level, the $\arm_i$ will be eliminated immediately if $\phat_{\arm_i}<\phat_{\tok^o}+\alpha$. If $\arm_i$ is eliminated, we increase the $j$ by $1$ and go for the next arm in the stream. The event that $\arm_i$ will be added to $\cA$ and $\tok^o$ will be eliminated from $\cA$ if and only if 1):$\arm_i$ is pulled more than $\tau_j$ times, and 2): its estimated mean is bigger than $\phat_{\tok^o}+\alpha$. If a new arriving arm is added to $\cA$, we rank the arms in $\cA$ with their estimated mean.    Once the algorithm goes through all the arms in the stream, it returns the indexes of arms in $\cA$.  The details are displayed in Algorithm~\ref{alg:topk}.

\begin{algorithm}[!t]
\label{alg:topk}
\caption{Streaming \KAI}
\KwIn {$k$, \eps, $\delta$, and a stream of $n$ arms.}
\KwOut {The indexes of $k$ arms.}
initialize:  $j\leftarrow 1$, $i\leftarrow 1$, $\cA=\emptyset$\; 
\For{each arriving $\arm_i$ ($i \leq k$)}
{ 
 pull $\arm_i$ $s_1$ times to obtain $\phat_{\arm_i}$\; 
 insert $\arm_i$ to $\cA$\; 
 }
 $\tok^o \leftarrow \arg\min_{\arm \in \cA} \phat_{\arm}$\;
%Let $\tok^o$ be the arm in $\cA$ with the minimum estimated mean\;
%rank the arms in $\cA$ with their expected reward and denote $\tok^o$ for the arm with the minimum estimated mean\; 

\For{each arriving $\arm_i$ ($i > k$)}
{
    % sample $\alpha$ from the following distribution:
    % $\Pr(\alpha = x) \gets
    % \begin{cases}
    % \frac{1}{\log j + 1}, & \textrm{if $x = \frac{\varepsilon}{4}$} \\
    % 1 - \frac{1}{\log j + 1}, & \textrm{if $x = \frac{\varepsilon}{2}$} 
    % \end{cases}
    % $\;
    set $\alpha$ as $\frac{\varepsilon}{4}$ with probability $\frac{1}{\log j + 1}$, and as $\frac{\varepsilon}{2}$ with other probability $1 - \frac{1}{\log j + 1}$\;
    $\ell \leftarrow 1$\; 
    \While{true}
    { 
        pull $\arm_i$ for $s_{\ell}-s_{\ell-1}$ times\;
        \If{$\phat_{\arm_i} \geq \phat_{\tok^o} + \alpha$ and $s_{\ell}>\tau_j$}  
        { 
            $\tok^o \leftarrow \arm_i$\;
            insert $\arm_i$ into $\cA$ and update $\tok^o$\; 
            %rank the arms in $\cA$ with their estimated mean and let $\tok^o$ be the arm with minimum estimated mean\; 
            $j\leftarrow 1$\;
            \textbf{break}\;
        }
        \ElseIf{$\phat_{\arm_i} < \phat_{\tok^o} + \alpha$}
        { 
            $j\leftarrow j + 1$\; 
            \textbf{break}\;
        }
        \Else
        { 
            $\ell\leftarrow \ell+1$;
        }
    }
}
\Return the indexes of the arms in $\cA$\;
\end{algorithm}

\section{Streaming BAI}
Previous sections study the problems with {\it instance independent} sample complexity. However, for particular problem instances, the sample complexity could be highly optimized. In this section, we consider the streaming BAI problem and {investigate} the {\it instance dependent} sample complexity.  

\subsection{The Algorithm}

Algorithm~\ref{alg:instance-dependent} presents the pseudo-code to address the streaming BAI problem. Notice that our algorithm borrows the existing Exponential-Gap-Eliminaion algorithm~\citep{karnin2013almost,chen2017towards} as a framework. We substitute the selection component in this framework for Algorithm \ref{alg:main} (Line~\ref{alg:line} in Algorithm~\ref{alg:instance-dependent}), which is the major modification. Algorithm~\ref{alg:instance-dependent} runs in multiple rounds. Suboptimal arms in the stream are eliminated round by round until only one arm remains. In the $r$-th round, we maintain a set $S_r$ of the total $n$ arms in the stream, and $\varepsilon$-BAI algorithm is adopted as a subroutine to return an \eps-best arm $\arm^o_r$ from $S_r$. We then compute its estimated mean $I_r$ by pulling $\arm^o_r$ for $\frac{2}{\varepsilon_r^2} \log (\frac{1}{\delta_r})$ times. At the $r$-th round, we set the additional budget $B_r=\frac{6|S_r|}{\varepsilon_r^2}\log \big(\frac{40}{\delta_r} \big)$. If $B_r>0$,  for each arriving $\arm_i$ in $S_r\setminus \{\arm^o_r\}$, we pull it in multiple iterations. At the $\ell$-th iteration, we pull it $\frac{2^{\ell}}{\varepsilon_r^2} \log (\frac{40 h^2}{\delta_r})$ times to get its estimated mean $\hat{p}_{i}^{\ell}(r)$. Budget $B_r$ is decreased by $\frac{2^{\ell}}{\varepsilon_r^2} \log (\frac{40 h^2}{\delta_r})$ and $s_i$ is increased by $\frac{2^{\ell}}{\varepsilon_r^2} \log (\frac{40}{\delta_r})$ accordingly. If $\hat{p}_{i}^{\ell}(r)<I_r-\varepsilon_r$ holds ($\varepsilon_r=2^{-r}/4$), we remove $\arm_i$ from $S_r$ and consider the next arm in the stream till the number of pulls exceeds $\frac{2}{\varepsilon_r^2}\log\big(\frac{40h^2}{\delta_r} \big)$. If $B_r\leq 0$, we pull it $\frac{2}{\varepsilon_r^2} \log (\frac{40 h^2}{\delta_r})$ times to get its estimated mean $\hat{p}_{i}(r)$. If $\hat{p}_{i}(r)<I_r-\varepsilon_r$ holds, $\arm_i$ is removed from $S_r$.
We continue to check following arriving arms till the end of the stream. This procedure is repeated until $S_r$ contains only one $arm$ whose index is returned as our final output. Notice that for each round, the stream is only visited $O(1)$ times.

%\keke{there are lots of newly defined notations in this algorithms?}

%it aims at eliminating the suboptimal arms with $\Delta_i>(1/2)^r$. Besides, Since the best arm is unknown, we use Algorithm~\ref{alg:main} to obtain an accurate approximation of the best arm. 

%\begin{algorithm}
%	\caption{ID-BAI\label{alg:instance-dependent}}
%	\KwIn{Parameter $\delta$ and a stream of $n$ arms.}
%	\KwOut{The index of an arm.}
%	Initialize $r\leftarrow 1$, $S_r=\{\arm_1,\arm_2,\cdots,\arm_n\}$\;
%	\While{$|S_r|>1$}
%	{
%		$\varepsilon_r \leftarrow 2^{-r}/4$, $\delta_r \leftarrow \delta/(40\cdot r^2)$, $h\leftarrow 1$\;
%		$\arm^o \leftarrow$ $\varepsilon$-BAI$(\eps_r,\delta_r, S_r)$\; \label{alg:line}
%		pull $\arm^o$ $\frac{2}{\varepsilon_r^2} \log (\frac{1}{\delta_r})$ times\;% and let $I_r$ be its estimated mean\; 
%		\For{each arriving $\arm_i \in S_r\setminus \{\arm^o\}$}
%		{
%			pull $\arm_i$ for $\frac{2}{\varepsilon_r^2} \log (\frac{40 h^2}{\delta_r})$ times\;%, and let $\hat{p}_{i}(r)$ be the estimated mean\; 
%			\If{$\phat_{\arm_i}<\phat_{\arm^o}-\varepsilon_r$}
%			{ 
%				remove $\arm_i$ from $S_r$\;
%			}
%			$h\leftarrow h+1$\;
%		}
%		$r\leftarrow r+1$\;
%	}
%	\Return the index of the arm in $S$\;
%\end{algorithm}
 \begin{algorithm}[!t]
	\caption{ID-BAI\label{alg:instance-dependent}}
	\KwIn{Parameter $\delta$ and a stream of arms.}
	\KwOut{The index of an arm.}
	Initialize $r\leftarrow 1$, $S_r=\{\arm_1,\arm_2,\cdots,\arm_n\}$\;
	\While{$|S_r|>1$}
	{
		$\varepsilon_r\gets2^{-r}/4$, $\delta_r\gets\delta/(40\cdot r^2)$, $h\leftarrow 1$\;
		$\arm^o_r$     $\leftarrow$  $\varepsilon$-BAI$(\eps_r,\delta_r, S_r)$\; \label{alg:line}
		pull $\arm^o_r$ $\frac{2}{\varepsilon_r^2} \log (\frac{1}{\delta_r})$ times and let $I_r$ be the estimated mean\; 
		$B_r\leftarrow \frac{6|S_r|}{\eps^2_r}\log \big(\frac{40}{\delta_r}\big)$\;
		\For{each arriving $\arm_i \in S_r\setminus \{\arm^o_r\}$}
		{
		\If{$B_r>0$}
		{
		$s_i\leftarrow 0$, $\ell \leftarrow 1$\;
		\While{$s_i\leq \frac{2}{\varepsilon^2_r}\log(\frac{40h^2}{\delta_r})$}
		{
			pull $\arm_i$ for $\frac{2^\ell}{\varepsilon_r^2} \log (\frac{40}{\delta_r})$ times, and let $\hat{p}_{i}^{\ell}(r)$ be the estimated mean\; \label{line-alg3-7}
			$B_r\leftarrow B_r- \frac{2^{\ell}}{\varepsilon_r^2} \log (\frac{40}{\delta_r})$\;
			$s_i \leftarrow s_i+\frac{2^{\ell}}{\varepsilon_r^2} \log (\frac{40}{\delta_r})$\;
			\If{$\hat{p}_{i}^{\ell}(r)<I_r-\varepsilon_r$}
			{ 
				remove $\arm_i$ from $S_r$\;
					$h\leftarrow h+1$\;
					break\;
			}
			$\ell\leftarrow \ell+1$\;
		}
		}
		\Else{	pull $\arm_i$ for $\frac{2}{\varepsilon_r^2} \log (\frac{40}{\delta_r})$ times, and let $\hat{p}_{i}(r)$ be the estimated mean\; \label{line-alg3-77}
			\If{$\hat{p}_{i}(r)<I_r-\varepsilon_r$}
			{ 
				remove $\arm_i$ from $S_r$\;
			}
		}
		}
		$r\leftarrow r+1$\;
	}
	\Return the index of the arm in $S_r$\;
\end{algorithm}

Our main result for Algorithm~\ref{alg:instance-dependent} is formalized in the following theorem.
\begin{theorem}\label{thm:id-BAI}
Given a stream of $n$ arms and confidence parameter $\delta \in (0, 1)$, Algorithm~\ref{alg:instance-dependent} identifies the optimal arm with probability at least $1-\delta$, in which case it takes expected  $O\big(\sum_{i=2}^{n}\frac{1}{\Delta_i^2} \log \big(\frac{1}{\delta}\log\frac{1}{\Delta_i} \big) \big)$  arm pulls and $O(\log \Delta_{2}^{-1})$ passes using a single-arm memory.  
\end{theorem}

Here, we present some high level ideas why $O(\log \Delta_2^{-1})$ passes would suffice for the correctness of Algorithm \ref{alg:instance-dependent}. We consider $\varepsilon_r$ in the following two cases. 

\spara{Case 1} $\varepsilon_r\in (\Delta_2/3,1]$. In each round, Algorithm \ref{alg:instance-dependent} costs at most 2 passes. Therefore, the number of total passes is $O(\log \Delta_2^{-1})$. 

\spara{Case 2} $\varepsilon_r\leq \Delta_2/3$. According to Theorem \ref{thm:main}, $\arm^o_r$ is an $\eps_r$-best arm.  From Hoeffding bound, we have $I_r\geq \mu_{\arm_r^o}-\frac{\eps_r}{2}$. Besides,  the budget $B_r$ ensures that with high probability  $\hat{\mu}_{\arm_i}\leq \mu_{\arm_i}+\frac{\eps_r}{2}$ (see the proofs for details). Therefore,
\begin{align*}
    I_r-\hat{\mu}_{\arm_i}& \geq \mu_{\arm_r^o}-\frac{\eps_r}{2}-\hat{\mu}_{\arm_i} \geq\mu_{\arm_r^o}-{\eps_r}-{\mu}_{\arm_i}  \notag \\
    & \geq \mu_{\arm^*}-{\mu}_{\arm_i} -2\varepsilon_r\geq \varepsilon_r,
\end{align*}
which indicates that with high probability, all suboptimal arms  will be eliminated in the current round when $\eps_r\leq \Delta_2/3$. Therefore, it takes $O(1)$ passes in such case. The detail proofs are in Appendix \ref{appendix-3}. 

Note that previous BAI algorithms~\citep{karnin2013almost,tianyuan2019efficient} run in $R$ rounds. In each round, the number of pulls of each arm is fixed. As a comparison, one can convert those algorithms into streaming algorithms in $R$ passes. In this regard, the previous best known algorithm~\citep{tianyuan2019efficient} will run in $\log^*(n)\cdot \log (1/\Delta_2)$ passes, which is inferior to our $\log (1/\Delta_2)$ passes. For sample complexity, our algorithm achieves the optimal instance-dependent sample complexity up to a $\log\log(1/\Delta_2)$ term, compared with the lower bound in \citet{chen2017towards}.
%\keke{Briefly talk about the contribution/point why we study this problem. Also highlight the above mentioned instance-dependent optimal sample complexity.}

\section{Additional Related Work}
We review the related work, excluding those \citep{assadi2020exploration,maiti2020streaming,falahatgaroptimal20} discussed in Section \ref{state-art}. The problem of best arm identification is mostly considered in a non-streaming setting in the literature. Normally, two types of sample complexity are considered: {\it instance-dependent complexity} and {\it instance-independent complexity}. 
 
\spara{Instance-independent arm selection} Existing work \citep{even2002pac, kalyanakrishnan2010efficient, cao2015top, tianyuan2019efficient} achieves the optimal worst sample complexity $\Omega(\frac{n}{\varepsilon^2}\log\frac{k}{\delta})$, matching the lower bound in \citep{kalyanakrishnan2012pac}. In addition, the recent work \citep{hassidim2020optimal} shows that the complexity of identifying an \eps-best arm can be reduced to $\frac{n}{2\eps^2}\log \frac{1}{\delta}$. However, all these algorithms require $\Theta(n)$ memory, which is  inferior to ours.  Recently, \citet{assadi2020exploration} show that one can modify an $r$-round algorithm to an $r$-memory algorithm. In this way, the previous best known algorithm can be modified to a $O(\log^*(n))$ memory algorithm, which is also inferior to ours.  
 
\spara{Instance-dependent arm selection} The instance-dependent
sample complexity is closely tied to the bandit instance and is superior to the instance-independent complexity for ‘easy’ bandit instances. Existing work \citep{karnin2013almost,jamieson2014lil,chen2017towards,chen2017nearly,tianyuan2019efficient,tao2019collaborative,chen2017adaptive} focuses on achieving the optimal instance-dependent sample complexity. The algorithm in \citep{karnin2013almost,jamieson2014lil}  achieves the sample complexity $O\big(\sum_{i=2}^{n}\frac{1}{\Delta_i^2}\log \big(\frac{1}{\delta}\log\frac{1}{\Delta_i} \big) \big)$. Furthermore, \citet{jamieson2014lil} prove a lower bound such that for some instances the BAI problem needs at least $\Omega\big(\sum_{i=2}^{n}\frac{1}{\Delta_i^2} \log \big(\frac{1}{\delta}\log\frac{1}{\Delta_i} \big) \big)$ samples.  Recently, \citet{chen2017towards} propose an instance-wise lower bound  and almost optimal upper bound for the BAI problem. In another line of the work, \citet{garivier2016optimal} present a constant optimal algorithm under the assumption $\delta\rightarrow 0$.   

\spara{Regret Minimization in Streaming Model}
\citep{LiauSPY18, ChaudhuriK19} study the streaming bandits for regret minimization. Both algorithms consume $O(1)$ memory and visit the stream multiple passes. Since we achieve different objectives , their regret bound is not directly comparable to our sample complexity bound. It is interesting to see whether our algorithm can help to improve their results.

\section{Conclusion}\label{sec:conclusion}

We study the streaming \eps-\topk arms identification (\eps-KAI) problem and the streaming BAI problem. For \eps-KAI, we propose the first algorithm that applies to any $k$ and achieves the optimal sample complexity using a single-arm memory without any explicit assumptions. For streaming BAI, we present a single-arm memory algorithm that achieves a near instance-dependent optimal sample complexity within $O(\log \Delta_2^{-1})$ passes.

 \section*{Acknowledgement}
We would like to thank Pan Xu for helpful discussions.  This work is supported by the Ministry of Education, Singapore, under Tier-2 Grant R-252-000-A70-112, and by the National Research Foundation, Singapore
under its AI Singapore Programme (AISG Award No: AISG-PhD/2021-01-004[T]).  The views and conclusions contained in this paper are those of the authors and should not be interpreted as representing any funding agencies.
\bibliography{streaming.bib}

\begin{thebibliography}{21}
\providecommand{\natexlab}[1]{#1}
\providecommand{\url}[1]{\texttt{#1}}
\expandafter\ifx\csname urlstyle\endcsname\relax
  \providecommand{\doi}[1]{doi: #1}\else
  \providecommand{\doi}{doi: \begingroup \urlstyle{rm}\Url}\fi

\bibitem[Assadi \& Wang(2020)Assadi and Wang]{assadi2020exploration}
Assadi, S. and Wang, C.
\newblock Exploration with limited memory: streaming algorithms for coin
  tossing, noisy comparisons, and multi-armed bandits.
\newblock In \emph{Proc. STOC}, pp.\  1237--1250, 2020.

\bibitem[Bertsimas \& Mersereau(2007)Bertsimas and
  Mersereau]{bertsimas2007learning}
Bertsimas, D. and Mersereau, A.~J.
\newblock A learning approach for interactive marketing to a customer segment.
\newblock \emph{Operations Research}, 55\penalty0 (6):\penalty0 1120--1135,
  2007.

\bibitem[Cao et~al.(2015)Cao, Li, Tao, and Li]{cao2015top}
Cao, W., Li, J., Tao, Y., and Li, Z.
\newblock On top-k selection in multi-armed bandits and hidden bipartite
  graphs.
\newblock In \emph{Proc. NeurIPS}, pp.\  1036--1044, 2015.

\bibitem[Chaudhuri \& Kalyanakrishnan(2019)Chaudhuri and
  Kalyanakrishnan]{ChaudhuriK19}
Chaudhuri, A.~R. and Kalyanakrishnan, S.
\newblock Regret minimisation in multi-armed bandits using bounded arm memory.
\newblock In \emph{Proc. AAAI}, pp.\  10085--10092, 2019.

\bibitem[Chen et~al.(2017{\natexlab{a}})Chen, Chen, Zhang, and
  Zhou]{chen2017adaptive}
Chen, J., Chen, X., Zhang, Q., and Zhou, Y.
\newblock Adaptive multiple-arm identification.
\newblock In \emph{International Conference on Machine Learning}, pp.\
  722--730. PMLR, 2017{\natexlab{a}}.

\bibitem[Chen et~al.(2017{\natexlab{b}})Chen, Li, and Qiao]{chen2017nearly}
Chen, L., Li, J., and Qiao, M.
\newblock Nearly instance optimal sample complexity bounds for top-k arm
  selection.
\newblock In \emph{Proc. AISTATS}, pp.\  101--110, 2017{\natexlab{b}}.

\bibitem[Chen et~al.(2017{\natexlab{c}})Chen, Li, and Qiao]{chen2017towards}
Chen, L., Li, J., and Qiao, M.
\newblock Towards instance optimal bounds for best arm identification.
\newblock In \emph{Proc. COLT}, pp.\  535--592, 2017{\natexlab{c}}.

\bibitem[Even-Dar et~al.(2002)Even-Dar, Mannor, and Mansour]{even2002pac}
Even-Dar, E., Mannor, S., and Mansour, Y.
\newblock Pac bounds for multi-armed bandit and markov decision processes.
\newblock In \emph{Proc. COLT}, pp.\  255--270, 2002.

\bibitem[Falahatgar et~al.(2020)Falahatgar, Orlitsky, and
  Pichapati]{falahatgaroptimal20}
Falahatgar, M., Orlitsky, A., and Pichapati, V.
\newblock Optimal sequential maximization one interview is enough!
\newblock In \emph{Proc. ICML}, pp.\  2975--2984, 2020.

\bibitem[Garivier \& Kaufmann(2016)Garivier and Kaufmann]{garivier2016optimal}
Garivier, A. and Kaufmann, E.
\newblock Optimal best arm identification with fixed confidence.
\newblock In \emph{Proc. COLT}, pp.\  998--1027, 2016.

\bibitem[Hassidim et~al.(2020)Hassidim, Kupfer, and
  Singer]{hassidim2020optimal}
Hassidim, A., Kupfer, R., and Singer, Y.
\newblock An optimal elimination algorithm for learning a best arm.
\newblock In \emph{Proc. NeurIPS}, 2020.

\bibitem[Jamieson et~al.(2014)Jamieson, Malloy, Nowak, and
  Bubeck]{jamieson2014lil}
Jamieson, K., Malloy, M., Nowak, R., and Bubeck, S.
\newblock lil’{UCB}: An optimal exploration algorithm for multi-armed
  bandits.
\newblock In \emph{Proc. COLT}, pp.\  423--439, 2014.

\bibitem[Jin et~al.(2019)Jin, Shi, Xiao, and Chen]{tianyuan2019efficient}
Jin, T., Shi, J., Xiao, X., and Chen, E.
\newblock Efficient pure exploration in adaptive round model.
\newblock In \emph{Proc. NeurIPS}, pp.\  6605--6614, 2019.

\bibitem[Kalyanakrishnan \& Stone(2010)Kalyanakrishnan and
  Stone]{kalyanakrishnan2010efficient}
Kalyanakrishnan, S. and Stone, P.
\newblock Efficient selection of multiple bandit arms: theory and practice.
\newblock In \emph{Proc. ICML}, pp.\  511--518, 2010.

\bibitem[Kalyanakrishnan et~al.(2012)Kalyanakrishnan, Tewari, Auer, and
  Stone]{kalyanakrishnan2012pac}
Kalyanakrishnan, S., Tewari, A., Auer, P., and Stone, P.
\newblock Pac subset selection in stochastic multi-armed bandits.
\newblock In \emph{Proc. ICML}, pp.\  655--662, 2012.

\bibitem[Karnin et~al.(2013)Karnin, Koren, and Somekh]{karnin2013almost}
Karnin, Z., Koren, T., and Somekh, O.
\newblock Almost optimal exploration in multi-armed bandits.
\newblock In \emph{Proc. ICML}, pp.\  1238--1246, 2013.

\bibitem[Liau et~al.(2018)Liau, Song, Price, and Yang]{LiauSPY18}
Liau, D., Song, Z., Price, E., and Yang, G.
\newblock Stochastic multi-armed bandits in constant space.
\newblock In \emph{Proc. AISTATS}, pp.\  386--394, 2018.

\bibitem[Maiti et~al.(2020)Maiti, Patil, and Khan]{maiti2020streaming}
Maiti, A., Patil, V., and Khan, A.
\newblock Streaming algorithms for stochastic multi-armed bandits.
\newblock \emph{arXiv preprint arXiv:2012.05142}, 2020.

\bibitem[Tao et~al.(2019)Tao, Zhang, and Zhou]{tao2019collaborative}
Tao, C., Zhang, Q., and Zhou, Y.
\newblock Collaborative learning with limited interaction: Tight bounds for
  distributed exploration in multi-armed bandits.
\newblock In \emph{Proc. IEEE FOCS}, pp.\  126--146, 2019.

\bibitem[Thompson(1933)]{thompson1933likelihood}
Thompson, W.~R.
\newblock On the likelihood that one unknown probability exceeds another in
  view of the evidence of two samples.
\newblock \emph{Biometrika}, 25\penalty0 (3/4):\penalty0 285--294, 1933.

\bibitem[Zhou et~al.(2014)Zhou, Chen, and Li]{zhou2014optimal}
Zhou, Y., Chen, X., and Li, J.
\newblock Optimal pac multiple arm identification with applications to
  crowdsourcing.
\newblock In \emph{Proc. ICML}, pp.\  217--225, 2014.

\end{thebibliography}
\bibliographystyle{icml2021}

\appendix
\newpage
\onecolumn
\section{Proof of Sample Complexity of Algorithm~\ref{alg:main}}\label{appendix-1}
\eat{ The analysis of sample complexity follows by following idea:
\begin{enumerate}
\item when the arriving arm $\arm_i$  satisfies $\mu_{\arm_i}\leq \phat_{\arm^o}+\frac{3\varepsilon}{8}$. Since $\alpha=\frac{\varepsilon}{2}$ with high probability, by Hoeffding inequality, with high probability, the $\arm_i$ will be eliminated within $O(\frac{1}{\eps^2}\log \frac{1}{\delta})$ pulls.  Taking expectation on $\alpha$, the expected number of pulls for $\arm_i$ is  $O({1}/{\varepsilon^2}\log ({1}/{\delta}))$;
\item when the arriving arm $\arm_i$  satisfies $\mu_{\arm_i}\geq \phat_{\arm^o}+\frac{3\varepsilon}{8}$. The number of pulls of $\arm_i$ is upper bounded by $\tau_j$. However, this will not happen over times. The reason is that  when $\alpha=\frac{\varepsilon}{4}$,  applying Hoeffding bound, with high probability $\arm_i$ will be the new $\arm^o$. This ensures that current $\arm^o$ can not eliminate too many arms with $\Omega(\tau_j)$ pulls.  
\end{enumerate}
Now, we start our proof.}

\begin{proof}[Proof of Sample Complexity of Algorithm~\ref{alg:main}]
Given $\arm^o(t)$, the aim is to bound the expected number of pulls of each arriving arm between the $t$-th and $t+1$-th best arm change. 
Let $A_{t,j}$ be the $j$-th passed arm between the $t$-th and $(t+1)$-th best arm change, and $N_{\arm_i}$ be the number of pulls of arm $\arm_i$. We consider two cases. For first case $ \mu_{A_{t,j}} \le \phat_{\arm^o(t)} + \frac{3\varepsilon}{8}$, we will prove that the expected number of pulls of $A_{t,j}$ is optimal, \ie $O(1/\eps^2 \log (1/\delta))$. For second case $ \mu_{A_{t,j}} > \phat_{\arm^o(t)} + \frac{3\varepsilon}{8}$, we will prove that with probability $\Pr(\alpha=\eps/4)/2$, $A_{t,j}$ will be the new $\arm^o$.

%and $A_t$ be the passed arms between $t$-th and $(t+1)$-th best arm change. Hence, from the definition, for $p\leq |A_t|$, $Q_{t,p}=A_{t,p}$. Let us consider following two cases. 

 \textbf{Case 1: $\mu_{A_{t,j}} \le \phat_{\arm^o(t)} + \frac{3\varepsilon}{8}$.} Let $X_t(j)$ be the event $\mu_{A_{t,j}} \le \phat_{\arm^o(t)} + \frac{3\varepsilon}{8}$. 
 Conditioned on event $\alpha=\varepsilon/2$, $X_{t}(j)$ and $\phat_{A_{t,j}} \leq  \mu_{A_{t,j}}+\frac{\varepsilon}{8}$, then  $\phat_{A_{t,j}} \leq  \mu_{A_{t,j}}+\frac{\varepsilon}{8} \leq \phat_{\arm^o(t)}+\frac{\varepsilon}{2}=\phat_{\arm^o(t)}+\alpha$ and $A_{t,j}$ will be eliminated. Let $\phat_{A_{t,j}}(\ell)$ be the estimated mean of  $A_{t,j}$ in the $\ell$-th round. Therefore,
 \begin{align}
     \Pr \bigg(N_{A_{t,j}}> s_{\ell} \ \bigg | \ \alpha=\frac{\varepsilon}{2},  X_{t}(j)\bigg)& \leq \Pr\bigg(\phat_{A_{t,j}}(\ell) \geq  \mu_{A_{t,j}}+\frac{\varepsilon}{8}\bigg) \notag \\
     & \leq \frac{2\delta}{C}\exp \big(-2^{\ell-1} \big) \notag \\
     & \leq \frac{1}{8^{\ell}}.
     \end{align}
 Note that $N_{A_{t,j}}\leq 2\tau_j$ and $\Pr(\alpha=\frac{\varepsilon}{4})=\frac{1}{\log j+1}$. We have
 \begin{align}
 \label{eq:samplecom-1}
    \EE[N_{A_{t,j}}\mid X_{t}(j)] & =\EE\bigg[N_{A_{t,j}}\ind\bigg(\alpha=\frac{\varepsilon}{2}\bigg) \ \bigg| \ X_{t}(j)\bigg]+\EE\bigg[N_{A_{t,j}}\ind\bigg(\alpha=\frac{\varepsilon}{4}\bigg)\ \bigg| \ X_{t}(j)\bigg] \notag \\
    & \leq  \sum_{\ell=1} \bigg( \frac{s_\ell}{8^{\ell-1}} \bigg)+\frac{2\tau_j}{ \log j+1} \notag \\
    & =O\bigg(\frac{1}{\varepsilon^2} \log \frac{1}{\delta}\bigg).
 \end{align}
 
 \textbf{Case 2: $ \mu_{A_{t,j}} > \phat_{\arm^o(t)} + \frac{3\varepsilon}{8}$.}  Let $X^c_t(j)$ be the event $ \mu_{A_{t,j}} > \phat_{\arm^o(t)} + \frac{3\varepsilon}{8}$. Then 
 \begin{align}
 \label{eq:samplecom-2}
   \EE[N_{A_{t,j}}\mid X_{t}^c(j)]\leq 2\tau_j=O\bigg(\frac{1}{\varepsilon^2}\log \frac{j^2}{\delta}\bigg).  \end{align}
 Note that conditioned on event $\alpha=\frac{\varepsilon}{4} $ and $X^c_t(j)$, if $\phat_{A_{t,j}}\geq \mu_{A_{t,j}}-\frac{\varepsilon}{8}$, then $\phat_{A_{t,j}} \geq \phat_{\arm^o(t)}+\frac{\varepsilon}{4}$ and thus $A_{t,j}=\arm^o(t+1)$.  
 Hence, conditioned on event $\alpha=\frac{\varepsilon}{4}$ and $X^c_t(j)$,
 \begin{align}
     \Pr\bigg(\phat_{A_{t,j}}(\ell)\leq \phat_{\arm^o(t)}+\frac{\varepsilon}{4} \bigg) & \leq \Pr\bigg(\phat_{A_{t,j}}(\ell)\leq \mu_{A_{t,j}}-\frac{\varepsilon}{8}\bigg) \notag \\
     & \leq \frac{2\delta}{C}\exp \big(-2^{\ell-1} \big) \leq \frac{1}{8^{\ell}}.
 \end{align}
Therefore, conditioned on event $\alpha=\frac{\varepsilon}{4}$ and $X^c_t(j)$,
 \begin{align}
 \label{eq:the1-M=i}
    \Pr\big(A_{t,j}=\arm^o(t+1)\big)\geq  1- \sum_{\ell=1}^{\infty}\Pr\bigg(\phat_{\arm^o(t)}\geq \phat_{A_{t,j}}(\ell)-\frac{\varepsilon}{4}\bigg)\geq \frac{1}{2}.
 \end{align}
 %$Recall that $Q_{t,p}$ is the $p$-th passed arm in the stream after $t$-th best arm changes.  Let $\cM_t$ be the set of passed arms between $t$-th and $(t+1)$-th best arm change. Hence $Q_{t,p}\in \cM_t$, for $p \leq |\cM_t|$. 
 Let 
  \begin{align}
  \label{def:St}
     \mathcal{S}_{t}=\bigg\{ Q_{t,r} \;\big|\; r> 32 \text{ and } \mu_{Q_{t,r}}>\phat_{\arm^o(t)}+\frac{3\varepsilon}{8} \bigg\}.
\end{align}
 %\begin{align}
 %    S_{t,p}=\big\{ Q_{t,r} \;\big|\; r\in [p], \text{ and }  \phat_{Q_{t,r}}<\mu_{\arm^o(t)}-\frac{3\varepsilon}{8} \big\}.
% \end{align}
%Note that 
% $\Pr(\alpha=\frac{\varepsilon}{4})=1/(\log j+1)$.   From~\eqref{eq:the1-M=i}, we have
% \begin{align}
 %\label{eq:samplecom-3}
 % \Pr(a_t>p)=\Pr(Q_{t,p}\in A_t) & \leq \prod_{r:Q_{t,r}\in S_{t,p}} \big( 1- \frac{1}{2} \frac{1}{(\log r+1)}\big) \notag \\
 % & \leq \prod_{r:Q_{t,r}\in S_{t,p}} \big( 1-  \frac{1}{2(\log p+1)}\big) \notag \\ 
%  & = \big( 1-  \frac{1}{2(\log p+1)}\big)^{|S_{t,p}|} \notag \\
 % & \leq 2^{-\frac{|S_{t,p}|}{2(\log p+1)}},
 %  \end{align}
  % where the lest inequality is due to $(1-x)^r\leq 2^{-rx}$.
%Given $\arm^o(t)$, let 
%\begin{align}
  %   r_m: \Pr(a_t>r_{m-1})\geq \frac{1}{2^m}, \text{ and } \Pr(a_t>r_m)\leq \frac{1}{2^m}.
%\end{align}
Now we are ready to bound the sample complexity. Let $A_t$ be the set of passed arms between $t$-th and $(t+1)$-th best arm change. Recall $Q_{t,r}$ is the $r$-th arriving arm after $t$-th best arm change.  Let $O_{t,r}$ be the event $Q_{t,r}\in A_t$.  We assume $Q_{t,1}=\arm_{i+1}$. Recall $n$ is the  number of arms. Let $W_t$ be the expected number of pulls of each arm between $t$-th and $(t+1)$-th best arm change. We have
\begin{align}
    W_t=\sum_{r=1}^{n-i}\Pr(O_{t,r}) \EE[N_{Q_{t,r}}\mid O_{t,r} ].
\end{align}
We will prove 
\begin{align}
\label{eq:sample-main}
    W_t\leq O\bigg(\frac{\log\delta^{-1}}{\varepsilon^2}\sum_{r=1}^{n-i}\Pr(O_{t,r})\bigg). \end{align}
If~\eqref{eq:sample-main} holds, then
the sample complexity of Algorithm~\ref{alg:main} is
\begin{align}
     &\sum_{i=1}^{n} \EE[N_{\arm_i}] \notag \\
    =& \EE[N_{\arm_1}]+\sum_{i=1}^{n}\sum_{r=1}^{n-i}\sum_{t=1}^{i} \EE[N_{Q_{t,r}}\ind(\arm^o(t)=\arm_{i},O_{t,r})] \notag \\
     =&\EE[N_{\arm_1}]+\sum_{i=1}^{n}\sum_{t=1}^{i} \bigg(\Pr(\arm^o(t)=\arm_{i})    \notag \\
     & \qquad \qquad \cdot \bigg(\sum_{r=1}^{n-i} \EE[N_{Q_{t,r}} \mid \arm^o(t)=\arm_{i},O_{t,r}] \Pr(O_{t,r} \mid \arm^o(t)=\arm_{i})\bigg) \bigg) \notag \\
     \leq &O\bigg(\frac{\log \delta^{-1}}{\varepsilon^2}\bigg)\sum_{i=1}^{n}\sum_{t=1}^{i} \bigg(\Pr(\arm^o(t)=\arm_{i})\sum_{r=1}^{n-i}\Pr(O_{t,r} \mid \arm^o(t)=\arm_{i})\bigg) \notag \\
     \leq &  O\bigg(\frac{\log \delta^{-1}}{\varepsilon^2}\bigg)\sum_{i=1}^{n}\sum_{r=1}^{n-i}\sum_{t=1}^{i} \bigg(\Pr(\arm^o(t)=\arm_{i})\Pr(O_{t,r} \mid \arm^o(t)=\arm_{i})\bigg) \notag \\
     = & O\bigg(\frac{\log \delta^{-1}}{\varepsilon^2}\bigg)\sum_{i=1}^{n}\sum_{r=1}^{n-i}\sum_{t=1}^{i} \bigg(\ind(\arm^o(t)=\arm_i,O_{t,r})\bigg) \notag \\
     = & O\bigg(\frac{n}{\varepsilon^2} \log\frac{1}{\delta}\bigg),
     \end{align}
    where we apply~\eqref{eq:sample-main} in first inequality.  
Now, we focus on proving~\eqref{eq:sample-main}. For $W_t$, we decompose it to
\begin{align*}
 & \sum_{r=1}^{n-i}\Pr(O_{t,r}) \EE[N_{Q_{t,r}}\mid O_{t,r}] \notag \\
 =&\sum_{r:Q_{t,r}\notin S_t}\Pr(O_r) \EE[N_{Q_{t,r}}\mid O_{t,r}]+\sum_{r:Q_{t,r}\in S_t}\Pr(O_{t,r}) \EE[N_{Q_{t,r}}\mid O_{t,r}] \notag \\
  =& 2\underbrace{\sum_{r:Q_{t,r}\notin S_t}\Pr(O_{t,r}) \EE[N_{Q_{t,r}}\mid O_{t,r}]}_{I_1}+\underbrace{\sum_{r:Q_{t,r}\in S_t}\Pr(O_{t,r}) \EE[N_{Q_{t,r}}\mid O_{t,r}]}_{I_2}  -\underbrace{\sum_{r:Q_{t,r}\notin S_t}\Pr(O_{t,r}) \EE[N_{Q_{t,r}}\mid O_{t,r}]}_{I_3}.
\end{align*}
{\bf Bounding term $I_1$:} If $\mu_{{Q_{t,r}}}\leq \phat_{\arm^o(t)}+\frac{3\varepsilon}{8}$,  from~\eqref{eq:samplecom-1}, $\EE[N_{Q_{t,r}}\mid O_{t,r}]=O\big(\frac{1}{\varepsilon^2}\log\frac{1}{\delta}\big)$. If $r<32$, $\EE[N_{Q_{t,r}}\mid O_{t,r}]\leq 2\tau_r=O\big(\frac{1}{\varepsilon^2}\log\frac{1}{\delta}\big)$. Combining above two cases together, we have $I_1=O\big(\frac{1}{\varepsilon^2}\log\frac{1}{\delta} \sum_{r:Q_{t,r}\notin S_t}\Pr(O_{t,r})\big)$.\\
{\bf Bounding term $I_2-I_3$:} We first decompose term $I_3$. Assume $\cS_{t}=\{Q_{t,r_1},Q_{t,r_2},\cdots \}$ and  $r_1<r_2<\cdots<r_m<\cdots$. Let $r_0=0$,
\begin{align*}
    T_m =(\Pr(O_{t,r_m})-\Pr(O_{t,r_{m}+1}))\sum_{x=1}^{m}\sum_{r=r_{x-1}+1}^{r_x-1} \EE[N_{Q_{t,r}}\mid O_{t,r}].
\end{align*}
Next, we show that $\sum_{m=1}^{|\cS_t|}T_m\leq I_3$.  We denote the coefficient of $\EE[N_{Q_{t,r}}\mid O_{t,r}]$ in $\sum_{m=1}^{|\cS_t|}T_m$ as $c_r$. Assume $r\in (r_{m-1},r_{m})$,  we have
\begin{align}
    c_r&=\sum_{x=m}^{|\cS_t|} \Pr(O_{t,r_x})-\Pr(O_{t,r_x+1}).
\end{align}
 Since $\Pr(O_{t,1})\geq \Pr(O_{t,2})\geq \Pr(O_{t,3})\geq\cdots$, we obtain $c_r\leq \Pr(O_{t,r_m})$. Note that for  $r\in (r_{m-1},r_{m})$, the coefficient of $\EE[N_{Q_{t,r}}\mid O_{t,r}]$ in $I_3$ is $\Pr(O_{t,r})$ and $\Pr(O_{t,r})\geq \Pr(O_{t,r_m})\geq c_r$, we have
 \begin{align}
 \label{eq:bounding-term-I_3}
     \sum_{m=1}^{|\cS_{t}|}T_m\leq I_3.
 \end{align}
Next, we decompose term $I_2$. Define
 \begin{align*}
   L_m=\Pr(O_{t,r_m})\EE[N_{Q_{t,r_m}}\mid O_{t,r_m}]-T_m-s_1(m-1)(\Pr(O_{t,r_m})-\Pr(O_{t,r_m+1})).
\end{align*}
We will show 
\begin{align}
\label{eq:samplecom-Lm}
    L_m\leq -\Pr(O_{t,r_m})s_1
\end{align}
holds for  $m\geq 1$.
 %From definition \begin{align*}
     % L_1& =\Pr(O_{t,r_1})\EE[N_{Q_{t,r_1}}\mid O_{t,r_1}]-T_1.
%\end{align*}
For $Q_{t,r}\in \cS_t$,
 \begin{align}
    \Pr(O_{t,r+1}\mid O_{t,r})&\leq \Pr\bigg(O_{t,r+1}\ \bigg| \ \alpha =\frac{\varepsilon}{4},O_{t,r}\bigg)\Pr\bigg(\alpha=\frac{\varepsilon}{4} \ \bigg| \ O_{t,r}\bigg)+\bigg(1-\Pr\bigg(\alpha=\frac{\varepsilon}{4} \ \bigg| \ O_{t,r}\bigg)\bigg) \notag \\
    & = \Pr\bigg(O_{t,r+1}\ \bigg| \ \alpha = \frac{\varepsilon}{4},O_{t,r}\bigg)\cdot \frac{1}{\log r+1}+\bigg(1-\frac{1}{\log r+1}\bigg) \notag \\
    &= \bigg(1-\Pr\bigg(Q_{t,r}=\arm^o(t+1)\ \bigg| \ \alpha=\frac{\varepsilon}{4},O_{t,r}\bigg)\bigg)\cdot \frac{1}{\log r+1}+\bigg(1-\frac{1}{\log r+1}\bigg) \notag \\
    & \leq 1-\frac{1}{2(\log r+1)},
\end{align}
 where the first equality is due to event $\alpha=\frac{\varepsilon}{4}$ is independent of $O_{t,r}$ and $\Pr(\alpha=\frac{\varepsilon}{4})=\frac{1}{\log r+1}$, the last inequality is due to~\eqref{eq:the1-M=i}. \\ 
 Therefore for $Q_{t,r}\in \cS_t$, 
\begin{align}
\label{eq:samplecom-boundOtr}
    \Pr(O_{t,r+1})=\Pr(O_{t,r+1}\mid O_{t,r})\Pr(O_{t,r})\leq \bigg(1-\frac{1}{2(\log r+1)}\bigg)\Pr(O_{t,r}).
\end{align}
  Since  each arm is pulled at least $s_1$ times, we have 
\begin{align}
\label{eq:samplecomtermI4}
    \frac{T_m}{(\Pr(O_{t,r_m})-\Pr(O_{t,r_m+1}))} &= \sum_{x=1}^m\sum_{r=r_{x-1}}^{r_x-1}\EE[N_{Q_{t,r}}|O_{t,r}] \notag \\
    & \geq s_1\sum_{x=1}^m (r_x-r_{x-1}-1) \notag \\
    &\geq (r_m-m)s_1.
\end{align}
Therefore,
%\begin{align}
%\label{eq:samplecom-boundL1}
 %  L_1= & \Pr(O_{t,r_1})\EE[N_{Q_{t,r_1}}\mid O_{t,r_1}]-T_1 \notag \\
  %  \leq &2\tau_{r_1}\Pr(O_{t,r_1})-(\Pr(O_{t,r_1})-\Pr(O_{t,r_1+1}))(r_1-1)s_1 \notag \\
   % \leq & 2\tau_{r_1}\Pr(O_{t,r_1})-\frac{1}{2(\log r_1+1)}(r_1-1)\Pr(O_{t,r_1})s_1 \notag \\
    %\leq & -\Pr(O_{t,r_1})s_1,
%\end{align}
%where the first inequality is due to~\eqref{eq:samplecomtermI4} and $\EE[N_{Q_{t,r_1}}\mid O_{t,r_1}]\leq 2\tau_{r_1}$, the second inequality is due to \eqref{eq:samplecom-boundOtr}, the last inequality is due to the factor that $r_1\geq 32$.
%Now, we assume $L_{x}\leq -\Pr(O_{t,r_x})s_1$ holds for all $x\leq m-1$ and want to prove $L_{m}\leq -\Pr(O_{t,r_m})s_1$. Since $L_{x}\leq -\Pr(O_{t,r_x})s_1$, then $L_{x}\leq -\Pr(O_{t,r_m})s_1$ for $x\leq m-1$. 
\begin{align}
\label{eq:samplecom-Lm'}
L_m=&\Pr(O_{t,r_m})\EE[N_{Q_{t,r_m}}\mid O_{t,r_m}]- T_m -s_1(m-1) (\Pr(O_{t,r_m})-\Pr(O_{t,r_m+1})) \notag \\
  \leq &  2\tau_{r_m}\Pr(O_{t,r_m})-s_1(r_m-m) (\Pr(O_{t,r_m})-\Pr(O_{t,r_m+1})) -s_1(m-1) (\Pr(O_{t,r_m})-\Pr(O_{t,r_m+1})) \notag \\
    = &  2\tau_{r_m}\Pr(O_{t,r_m})- s_1(r_m-1)(\Pr(O_{t,r_m})-\Pr(O_{t,r_m+1}))  \notag \\ 
    \leq & 2\tau_{r_m}\Pr(O_{t,r_m})-\frac{s_1}{2(\log r_m+1)}(r_m-1)\Pr(O_{t,r_m}) \notag \\
    \leq & -\Pr(O_{t,r_m})s_1,
\end{align}
where the first inequality is due to~\eqref{eq:samplecomtermI4} and the factor $N_{Q_{t,r}}\leq 2\tau_r$ (equation~\eqref{eq:samplecom-2}), the third inequality is due to \eqref{eq:samplecom-boundOtr}, the last inequality is due to the factor that $r_m\geq 32$ (see the definition of $\cS_t$ and $Q_{t,r_m}$).
Note that 
\begin{align}
\label{eq:finalsam1}
    \Pr(O_{t,r_m})\geq \sum_{x=r_m}^{\infty} \big(\Pr(O_{t,x})-\Pr(O_{t,x+1})\big)\geq \sum_{x=m}^{|\cS_t|} \big(\Pr(O_{t,r_x})-\Pr(O_{t,r_x+1})\big),
\end{align}
where $\cS_t$ is defined in~\eqref{def:St}.
We obtain $I_2-I_3$
\begin{align*}
& = \sum_{m=1}^{|\cS_t|}\Pr(O_{t,r_m}) \EE[N_{Q_{t,r_m}}\mid O_{t,r_m}]  -I_3 \notag \\
& \leq \sum_{m=1}^{|\cS_t|}\Pr(O_{t,r_m}) \EE[N_{Q_{t,r_m}}\mid O_{t,r_m}] -\sum_{m=1}^{|\cS_t|} T_m \notag \\
& = \sum_{m=2}^{|\cS_t|}\Pr(O_{t,r_m}) \EE[N_{Q_{t,r_m}}\mid O_{t,r_m}] -\sum_{m=2}^{|\cS_t|} T_m+L_1 \notag \\
& \leq \sum_{m=2}^{|\cS_t|}\Pr(O_{t,r_m}) \EE[N_{Q_{t,r_m}}\mid O_{t,r_m}] -\sum_{m=2}^{|\cS_t|} T_m-s_1\sum_{m=2}^{|\cS_t|}\Pr(O_{t,r_m})-\Pr(O_{t,r_{m}+1}) \notag \\
& = \sum_{m=3}^{|\cS_t|}\Pr(O_{t,r_m}) \EE[N_{Q_{t,r_m}}\mid O_{t,r_m}] -\sum_{m=3}^{|\cS_t|} T_m-s_1\sum_{m=3}^{|\cS_t|}\Pr(O_{t,r_m})-\Pr(O_{t,r_{m}+1})+L_2 \notag \\
&\leq \sum_{m=3}^{|\cS_t|}\Pr(O_{t,r_m}) \EE[N_{Q_{t,r_m}}\mid O_{t,r_m}] -\sum_{m=3}^{|\cS_t|} T_m-2s_1\sum_{m=3}^{|\cS_t|}\Pr(O_{t,r_m})-\Pr(O_{t,r_{m}+1}) \notag  \\
& = \cdots \leq \cdots =\cdots \leq \cdots \notag \\ 
& =\Pr(O_{t,r_{|\cS_t|}}) \EE[N_{Q_{t,r_{|\cS_t|}}}\mid O_{t,r_{|\cS_t|}}]-T_{|\cS_t|}-s_1(|\cS_t|-1)(\Pr(O_{t,r_{|\cS_t|}})-\Pr(O_{t,r_{|\cS_t|}+1})) \notag \\
& =L_{|\cS_t|} \leq 0,
  \end{align*}
  where the first inequality is from \eqref{eq:bounding-term-I_3}, the second inequality is due to \eqref{eq:finalsam1} and $L_1\leq -\Pr(O_{t,r_1})s_1$ from \eqref{eq:samplecom-Lm'}. Therefore $W_t=2I_1+I_2-I_3\leq 2I_1 \leq O\big(\frac{\log\delta^{-1}}{\varepsilon^2}\sum_{r=1}^{n-i}\Pr(O_{t,r})\big)$, which competes the proof.
\end{proof}

\section{Proofs of Theorem \ref{thm:top-k}.}
\label{appendix-2}

We refer to the event that $\cA$ is added a new arm (in Line 2 or 11 of Algorithm~\ref{alg:topk}) as a \topk arm change. We use $\tok^o(t)$ to denote the arm with the minimum expected reward in $\cA$ after the $t$-th \topk arm change. Hence, $\tok^o(1)=\arm_1$. We denote $\arm^o(t)$ for the arm that is added to $\cA$ for $t$-th \topk arm change.  Besides, let $\cA_t$ be the version of $\cA$ after $t$-th \topk arm change.  Assume that the last arm inserted to $\cA$ is $\arm^o(T)$. Hence, the returned version of $\cA$ is $\cA_T$. We  use $\arm^*(k)$ to denote $k$-th largest arm. 

Similarly, the proof of correctness consists of two parts. 
In the first part, we establish the relation between all $\arm^o$ and $\arm^*(k)$ in Lemma~\ref{lem:main-1-topk}. In the second part, we then complete the correctness proof based on the result of Lemma~\ref{lem:main-1-topk}.

\begin{lemma}
\label{lem:main-1-topk}
For any $\delta\in(0,1)$, it holds that
\begin{align*}
    \Pr\bigg(\cap_{t\geq1}\bigg\{\bigg\{\phat_{\arm^o(t)}< \mu_{\arm^*(k)}-\frac{5\varepsilon}{8}\bigg\}\cup\bigg\{\mu_{\arm^o(t)}\geq \mu_{\arm^*(k)}-\varepsilon\bigg\}\bigg\}\bigg)\geq 1-3/4\delta.
\end{align*}
\end{lemma}
\begin{proof}
\revise{
 We let $Q_{t,p}$ be the $p$-th passed arm after $t$-th \topk arm change.  Define $s(p):=s_{\ell}, s_{\ell-1}\leq \tau_p<s_{\ell}$. Consider the following virtual process: when algorithm~\ref{alg:topk} ends, if $Q_{t,p}$ is pulled less than $s(p)$ times, we pull $Q_{t,p}$ again. We pull $Q_{t,p}$ total $s(p)$ times (Noted this is the virtual sampling process). Hence, for all $p\geq 1$, $Q_{t,p}$ will be pulled exactly $s(p)$ times. We use $\phat'_{Q_{t,p}}$ for the estimated mean of $Q_{t,p}$ when it has been pulled $s(p)$ times.  If ${\arm^o(t+1)}=Q_{t,p}$, then $\phat^\prime_{Q_{t,p}}= \phat_{Q_{t,p}}$ holds by definition. 
Let 
\begin{align*}
    \cQ_d(T)=\bigg\{Q_{t,p}: \mu_{Q_{t,p}}\in \bigg(\mu_{\arm^*(k)}-\frac{d\varepsilon}{8}, \mu_{\arm^*(k)}-\frac{(d-1)\varepsilon}{8}\bigg],  p\geq 1, \text{ and } t\in [T] \bigg\},
\end{align*}
where $d$ is an integer and $d\geq 1$.  Fixed a $d\geq 9$, we consider the following cases. 
\\
\textbf{Case 1:} 
Let $G_{d}$ be the event that there exists $t$ and $p$ such that $\hat{\mu}_{\tok^o(t)}\leq \mu_{\arm^*}-\frac{d\eps}{4} -\frac{\eps}{2}$, $Q_{t,p}\in \cQ_{d}(T)$.
When $G_{d}$  is true, 
let $E_{d1}$ be the event that 1: $G_{d}$ occurs, and 2: for the first $t$ and $p$ such that $\hat{\mu}_{\tok^o(t)}\leq \mu_{\arm^*}-\frac{d\eps}{4} -\frac{\eps}{2}$ and $Q_{t,p}\in \cQ_{d}(T)$, it holds that $\arm^o(t+1)=Q_{t,p}$ and $\hat{\mu}_{\arm^{o}(t+1)}\in [\mu_{\arm^o(t+1)}-(d-8)\eps/8, \mu_{\arm^o(t+1)}+(d-8)\eps/8]$.  We note that according to the update rule of the algorithm, if for all $p\geq 1$, $| \phat'_{Q_{t,p}}-{\mu}_{Q_{t,p}} |\leq  \frac{(d-8)\varepsilon}{8}$, then $Q_{t,p}$ will be $\tok^o(t+1)$. Define $F_{t,p}$ be the union of history till $Q_{t,p}$ comes. 
Therefore,
\begin{align}
\label{eq:kEd1}
 \min_{F_{t,p}: G_{d}  \text{ is true under } F_{t,p}}  \PP\bigg\{E_{d1} \ \bigg | \ F_{t,p} \bigg \} 
  & \geq 1-\sum_{p\geq 1} \Pr\bigg(| \phat'_{Q_{t,p}}-{\mu}_{Q_{t,p}} |\leq  \frac{(d-8)\varepsilon}{8}  \bigg) \notag \\
  & \geq 1- \frac{2\delta}{p^2 kC^{d-8}} \notag \\
  & \geq 1- \frac{4\delta}{kC^{d-8}}.
\end{align}
\textbf{Case 2:} 
Let $G_{d}(t)$ be the event that $\hat{\mu}_{\tok^o(t)}> \mu_{\arm^*}-\frac{d\eps}{4} -\frac{\eps}{2}$. 
Let $E_{d2}(t)$ be the event that 1: $G_{d}(t)$ is true, 2: for any $p\geq 1$ and $Q_{t,p}\in\cQ_{d}(T)$, $| \phat'_{Q_{t,p}}-{\mu}_{Q_{t,p}} |< \frac{(d-8)\varepsilon}{8}$. We have
\begin{align}
\label{eq:kEd2}
\PP\bigg\{ E_{d2}(t) \ \bigg | \ G_{d}(t)\bigg\}%& \geq 1-\PP(E_{d2}^c(t)) \notag \\
    & \geq 1-\PP\bigg\{ \exists p
\geq 1: Q_{t,p} \in \cQ_{d}(T), | \phat'_{Q_{t,p}}-{\mu}_{Q_{t,p}} |\geq  \frac{(d-8)\varepsilon}{8}
 \bigg \} \notag \\
  & \geq 1- \sum_{p\geq 1}\Pr\bigg(| \phat'_{Q_{t,p}}-{\mu}_{Q_{t,p}} |\geq  \frac{(d-8)\varepsilon}{8}  \bigg) \notag \\
  & \geq 1-\sum_{p\geq 1}\frac{2\delta}{p^2k C^{d-8}} \notag \\
  & \geq 1- \frac{4\delta}{kC^{d-8}}.
\end{align}
When $G_{d}(t)$ and $E_{d2}(t)$ are true and $\arm^o(t+1)\in \cQ_{d}(T)$, we have
\begin{align*}
    | \phat_{\arm^o(t+1)}-{\mu}_{\arm^o(t+1)} |< \frac{(d-8)\varepsilon}{8}.
\end{align*}
Fixed a $d<9$, we consider the following cases. 
\\
\textbf{Case 3:} 
Let $E_{d3}$ be the event that 1: $G_{d}$ occurs, and 2: for the first $t$ and $p$ such that $\hat{\mu}_{\tok^o(t)}\leq \mu_{\arm^*}-\frac{d\eps}{4} -\frac{\eps}{2}$ and $Q_{t,p}\in \cQ_{d}(T)$, $\arm^o(t+1)=Q_{t,p}$ and $\hat{\mu}_{\arm^{o}(t+1)}\in [\mu_{\arm^o(t+1)}-d\eps/8, \mu_{\arm^o(t+1)}+d\eps/8]$. 
 We have
\begin{align}
\label{eq:kEd3}
  \min_{F_{t,p}: G_{d}  \text{ is true under } F_{t,p}}   \PP\bigg\{ E_{d3} \ \bigg | \ G_{d} \ \bigg \} & \geq  1-\sum_{p\geq 1}\Pr\bigg(| \phat'_{Q_{t,p}}-{\mu}_{Q_{t,p}} |\leq  \frac{d\varepsilon}{8} \bigg) \notag \\
  & \geq 1- \sum_{p\geq 1}\frac{2\delta}{p^2k C^{d}} \notag \\
  & \geq 1- \frac{4\delta}{kC^{d}}.
\end{align}
\textbf{Case 4:} 
Let $E_{d4}(t)$ be the event that 1: $G_{d}(t)$ is true, 2: for any $p\geq 1$ and $Q_{t,p}\in\cQ_{d}(T)$, $| \phat'_{Q_{t,p}}-{\mu}_{Q_{t,p}} |< \frac{d\varepsilon}{8}$. Then, we have
\begin{align}
\label{eq:kEd4}
   \PP\bigg\{ E_{d4}(t) \ \bigg | \ G_{d}(t)\bigg\}%& \geq 1-\PP(E_{d2}^c(t)) \notag \\
     \geq &1- \PP\bigg\{ \exists p
\geq 1: Q_{t,p} \in \cQ_{d}(T), | \phat'_{Q_{t,p}}-{\mu}_{Q_{t,p}} |\geq  \frac{d\varepsilon}{8}
 \bigg \} \notag \\
  & \geq 1- \sum_{p\geq 1}\Pr\bigg(| \phat'_{Q_{t,p}}-{\mu}_{Q_{t,p}} |\geq  \frac{d\varepsilon}{8}  \bigg) \notag \\
  & \geq 1-\sum_{p\geq 1}\frac{2\delta}{p^2k C^{d}} \notag \\
  & \geq 1- \frac{4\delta}{kC^{d}}.
\end{align}
When $G_{d}(t)$ and $E_{d4}(t)$ are true and $\arm^o(t+1)\in \cQ_{d}(T)$, then 
\begin{align*}
    | \phat_{\arm^o(t+1)}-{\mu}_{\arm^o(t+1)} |< \frac{d\varepsilon}{8}.
\end{align*}
Define $I$ as a random variable such that
$I=1$ if the following conditions hold.
\begin{enumerate}
    \item For all $d\geq 9$, if $G_{d}$ is true,  so does $E_{d1}$.
    \item For all $d\geq 9$ and $t\in [T]$, if $G_{d}(t)$ is true, so does $E_{d2}(t)$.
    \item For all $d\in [1,8]$, if $G_{d}$ is true, so does $E_{d3}$.
    \item For all $d\in [1,8]$ and $t\in [T]$, if $G_{d}(t)$ is true, so does $E_{d4}(t)$.
\end{enumerate}
Otherwise, $I=0$.  We use $E^c$ for the opposite event of $E$. Similar to the chain rule, we sequentially apply the \eqref{eq:kEd1},\eqref{eq:kEd2},\eqref{eq:kEd3}, and \eqref{eq:kEd4} to compute the probability of $I=1$. 
We have 
\begin{align*}
    \PP(I=1)& \geq  \EE\Bigg[ \prod_{d\geq 9}\bigg( \ind\{G_{d}^c\}+ \ind\{G_{d}\}\PP(E_{d1}\mid G_{d}) \bigg) \prod_{d\geq 9}  \prod_{t\in [T]} \bigg(  \ind\{G_{d}^c(t)\}+\ind\{G_{d}(t)\}\PP(E_{d2}\mid G_{d}(t)) \bigg)\notag \\
    & \quad \times \prod_{d< 9}\bigg( \ind\{G_{d}^c\}+ \ind\{G_{d}\} \PP(E_{d3}\mid G_{d})\bigg) \prod_{d< 9}  \prod_{t\in [T]} \bigg(  \ind\{G_{d}^c(t)\}+\ind\{G_{d}(t)\}\PP(E_{d4}\mid G_{d}(t)) \bigg)  \Bigg]  \notag \\
   &  \geq \prod_{d\geq 9}  \min_{F_{t,p}: G_{d}  \text{ is true under } F_{t,p}}\PP(E_{d1}\mid F_{t,p}) \cdot  \EE\Bigg[\prod_{d\geq 9} \prod_{t\in [T]} \bigg(  \ind\{G_{d}^c(t)\}+\ind\{G_{d}(t)\}\PP(E_{d2}\mid G_{d}(t)) \bigg)\Bigg] \notag \\
    & \quad \times \prod_{d< 9}  \min_{F_{t,p}: G_{d}  \text{ is true under } F_{t,p}} \PP(E_{d3}\mid F_{t,p}) \cdot \EE \Bigg[\prod_{d< 9} \prod_{t\in [T]} \bigg(  \ind\{G_{d}^c(t)\}+\ind\{G_{d}(t)\}\PP(E_{d4}\mid G_{d}(t)) \bigg)\Bigg],
\end{align*}
where the second inequality is due to \begin{align*}
    \EE\Bigg[ \prod_{d\geq 9}\bigg( \ind\{G_{d}^c\}+ \ind\{G_{d}\}\PP(E_{d1}\mid G_{d}) \bigg) \Bigg]\geq   \EE\Bigg[ \prod_{d\geq 9}\bigg( \ind\{G_{d}\}\PP(E_{d1}\mid G_{d}) \bigg) \Bigg] \geq  \prod_{d\geq 9}  \min_{F_{t,p}: G_{d}  \text{ is true under } F_{t,p}}\PP(E_{d1}\mid F_{t,p}).
\end{align*}
Note that $\PP(E_{d2}\mid G_{d}(t))$ can be bounded by \eqref{eq:kEd2} and $\PP(E_{d4}\mid G_{d}(t))$ can be bounded by \eqref{eq:kEd4}. We present an important property of $\tok^o(t)$. %
 For $t\geq k$, $\tok^o(t),\tok^o(t+1),\tok^o(t+2),\cdots,\tok^o(t+k)$ are distinct. Since $\alpha\geq \frac{\varepsilon}{4}$ holds, from the Algorithm, we have
$\phat_{\arm^o(t+r)}\geq\phat_{\tok^o(t+r-1)}+\frac{\varepsilon}{4}$. Till $(t+k)$-th \topk arm change, we have added more than $k$ arms, and for each added $\arm_i$ in $\cA$, we have $\phat_{\arm_i}\geq \phat_{\tok^o(t)}+\frac{\varepsilon}{4}$. Hence,   
$\phat_{\tok^o(t+k)}\geq \phat_{\tok^o(t)}+\frac{\varepsilon}{4}$ holds. For ease of exposition, we assume $\mu_{\arm_1}\leq \mu_{\arm_2}\leq \cdots \leq \mu_{\arm_k}$. Obviously, for $t<k$, we obtain $\phat_{\tok^o(t+k)}\geq \phat_{\tok^o(t)}+\frac{\varepsilon}{4}$. Therefore, for $t\geq 1$,
\begin{align}
    \label{eq:lem2topkbound}
    \phat_{\tok^o(t+k)}\geq \phat_{\tok^o(t)}+\frac{\varepsilon}{4}.
\end{align}
Assume
$I=1$. We now determine the number of occurrences of $G_{d}(t)$. For $I=1$,
 we have $\hat{\mu}_{\tok^{o}(T)}\leq U:= \mu_{\arm^*(k)}+\epsilon/8$. From \eqref{eq:lem2topkbound},  $\phat_{\tok^o(t+k)}\geq \phat_{\tok^o(t)}+\frac{\varepsilon}{4}$.  This means that the number of $\arm^o(t)$ with $\hat{\mu}_{\arm^o(t)}>L:=\mu_{\arm^*(k)}-\frac{d\eps}{4}-\frac{\eps}{2}$ is at most 
\begin{align}
  \frac{4k(U-L)}{\eps}\leq (d+3)k.
\end{align}
Therefore,
\begin{align}
    \Pr(I=1) & \geq \prod_{d=1}^{8} \bigg(1-\frac{4\delta}{kC^{d}} \bigg)\prod_{d=1}^{8} \bigg(1-\frac{4(d+3)\delta}{kC^{d}} \bigg)\prod_{d=9}^{\infty} \bigg(1-\frac{4\delta}{kC^{d-8}}\bigg)\prod_{d=9}^{\infty} \bigg(1-\frac{4(d+3)\delta}{kC^{d-8}} \bigg) \notag \\
    &\geq 1-\sum_{d=1}^{8}\frac{4(d+4)\delta}{C^{d}}-\sum_{d=9}^{\infty} \frac{4(d+4)\delta}{C^{d-8}} \notag\\
    &\geq 1-\frac{3\delta}{4}.
    \end{align}
where the last inequality is due to $C\geq 100$. 
Let $r_t \in \NN$ and satisfies
\begin{align}
\label{eq:clrk}
 \mu_{\arm^o(t)}\in \bigg(\mu_{\arm^*(k)}-\frac{r_t\eps}{8}, \mu_{\arm^*(k)}-\frac{(r_t-1)\eps}{8} \bigg].
\end{align}
Based on the definition of $I$ and $\cQ_d(T)$,  we have 
\begin{align}
    \{I=1\} & \subseteq  \bigg\{\bigcap_{t\geq1} \bigg\{ \bigg\{|\phat_{\arm^o(t)}-\mu_{\arm^o(t)}|\leq\frac{(r_t-8)\varepsilon}{8}\bigg\}\notag \\
    &\qquad \qquad \bigcup   \bigg\{\mu_{\arm^o(t)}\geq \mu_{\arm^*(k)}-\varepsilon\bigg\} \bigg\}\bigg\} \notag \\
     & \subseteq   \bigg\{\bigcap_{t\geq1} \bigg\{\bigg\{\phat_{\arm^o(t)}< \mu_{\arm^*(k)}-\frac{5\varepsilon}{8}\bigg\} \notag \\
    & \qquad \qquad \bigcup\bigg\{\mu_{\arm^o(t)}\geq \mu_{\arm^*(k)}-\varepsilon\bigg\}\bigg\} \bigg\},
\end{align}
where the second formula follows since if $|\phat_{\arm^o(t)}-\mu_{\arm^o(t)}|\leq\frac{(r_t-8)\varepsilon}{8}$, we have
\begin{align}
    &\phat_{\arm^o(t)}\leq \mu_{\arm^o(t)}+\frac{(r_t-8)\varepsilon}{8} \notag \\
    &\leq \mu_{\arm^*(k)}-\frac{(r_t-1)\varepsilon}{8}+ \frac{(r_t-8)\varepsilon}{8} < \mu_{\arm^*(k)}-\frac{5\varepsilon}{8}.
\end{align}
This completes the proof. }
\end{proof}

 \begin{proof}[Proof of Theorem~\ref{thm:top-k}]
\textbf{Proof of Correctness}: 
For $\arm^*(s)$ ($s\in [k]$), we consider two cases. {\it Case I}: For $s\in [k]$, all $\arm^*(s)\in \cA_{T}$. Then the returned set $\cA_T$ is exactly \topk arms. {\it Case II}: Assume $\arm^*(s)\notin \cA_{T}$. Then there exists a $t$ such that 
\begin{align}
\label{eq:theorem-topk-1}
    \phat_{\arm^*(s)}\leq \phat_{\tok^o(t)}+\frac{\varepsilon}{2}. 
\end{align}
From Proposition~\ref{prop:chernoff} and union bound, 
\begin{align}
\label{eq:the1-bound-muotopk}
\Pr\bigg(\phat_{\arm^*(s)} \geq \mu_{\arm^*(s)} -\frac{\varepsilon}{8} \bigg)\geq 1-2k\exp \bigg(-\frac{s_1\varepsilon^2}{8} \bigg)\geq 1- \frac{\delta}{4}.
\end{align}
From \eqref{eq:theorem-topk-1},
\begin{align}
\label{eq:main-correct-2}
    \phat_{\tok^o(T)}\geq\phat_{\tok^o(t)}\geq  \phat_{\arm^*(s)} -\frac{\varepsilon}{2}.
\end{align} 
Combining~\eqref{eq:the1-bound-muotopk}~\eqref{eq:main-correct-2}  together, we have
\begin{align}
\label{eq:proof-correcttopk}
    \Pr\bigg(\phat_{\arm^o(T)}\geq \mu_{\arm^*(k)}-\frac{5\varepsilon}{8}\bigg)\geq 1-\frac{\delta}{4}.
\end{align}
From Lemma~\ref{lem:main-1}, we obtain
\begin{align}
    \Pr\bigg(\bigg\{\phat_{\arm^o(T)}< \mu_{\arm^*(k)}-\frac{5\varepsilon}{8}\bigg\}\bigcup\bigg\{\mu_{\arm^o(T)}\geq \mu_{\arm^*(k)}-\varepsilon\bigg\}\bigg) \geq 1-\frac{3\delta}{4}.
\end{align}
Let $A=\bigg\{\phat_{\arm^o(T)}\geq \mu_{\arm^*(k)}-\frac{5\varepsilon}{8}\bigg\}$ and $B=\bigg\{\mu_{\arm^o(T)}\geq \mu_{\arm^*(k)}-\varepsilon\bigg\}$. Then from~\eqref{eq:proof-correcttopk},  $\Pr(A)\geq 1-\frac{\delta}{4}$. From~\eqref{eq:proof-correcttopk}, $\Pr(A\cap B)\leq \frac{3\delta}{4}$.  Therefore $\Pr(B)\geq \Pr(A)- \Pr(A \cap B)\geq 1-\delta$.

\textbf{Proof of Sample Complexity}: 
the proof of sample complexity for Algorithm~\ref{alg:topk} is almost same as the proof of sample complexity for Algorithm~\ref{alg:main}.  We use the same notations used in the proof of sample complexity in Theorem~\ref{thm:main}. We also firstly consider two cases: {\it case I}: $\phat_{\tok^o(t)}\geq \mu_{A_{t,j}}-\frac{3\varepsilon}{8}$, {\it case II}: $\phat_{\tok^o(t)}< \mu_{A_{t,j}}-\frac{3\varepsilon}{8}$. Similar to the proof of Theorem~\ref{thm:main}, we obtain 
\begin{align}
\label{eq:final1}
    \EE[N_{A_{t,j}}\mid X_{t}(j)]=O\bigg(\frac{1}{\varepsilon^2}\log \frac{k}{\delta}\bigg),
\end{align}
and conditioned on event $\{\alpha=\frac{\varepsilon}{4}\}$ and $X^c_t(j)$
\begin{align}
\label{eq:final2}
    \EE[A_{t,j}=\arm^o(t+1)]\geq \frac{1}{2}.
\end{align}
After $t$-th arm change, assume the arrived arm is $\arm_{i+1}$, using~\eqref{eq:final1}, \eqref{eq:final2} and following the proof of bounding term $W_t$ in Theorem~\ref{thm:main} (only need to replace $\phat_{\arm^o(t)}$ with $\phat_{\tok^o(t)}$), we can obtain
\begin{align}
\label{eq:sample-topk}
    W_t\leq O\bigg(\frac{\log(k/\delta)}{\varepsilon^2}\sum_{r=1}^{n-i}\Pr(O_{t,r})\bigg).   
\end{align}
Similarly, the sample complexity of Algorithm~\ref{alg:topk} is
\begin{align}
     &\sum_{i=1}^{n} \EE[N_{\arm_i}] \notag \\
    =& \EE[N_{\arm_1}]+\sum_{i=1}^{n}\sum_{r=1}^{n-i}\sum_{t=1}^{i} \EE[N_{Q_{t,r}}\ind(\arm^o(t)=\arm_{i},O_{t,r})] \notag \\
     =&\EE[N_{\arm_1}]+\sum_{i=1}^{n}\sum_{t=1}^{i} \bigg(\Pr(\arm^o(t)=\arm_{i})  \notag \\
     & \qquad \qquad \cdot \bigg( \sum_{r=1}^{n-i} \EE[N_{Q_{t,r}} \mid \arm^o(t)=\arm_{i},O_{t,r}] \Pr(O_{t,r} \mid \arm^o(t)=\arm_{i})\bigg) \bigg) \notag \\
     \leq &O\bigg(\frac{\log (k/\delta)}{\varepsilon^2}\bigg)\sum_{i=1}^{n}\sum_{t=1}^{i} \bigg(\Pr(\arm^o(t)=\arm_{i})\sum_{r=1}^{n-i}\Pr(O_{t,r} \mid \arm^o(t)=\arm_{i})\bigg) \notag \\
     \leq &  O\bigg(\frac{\log (k/\delta)}{\varepsilon^2}\bigg)\sum_{i=1}^{n}\sum_{r=1}^{n-i}\sum_{t=1}^{i} \bigg(\Pr(\arm^o(t)=\arm_{i})\Pr(O_{t,r} \mid \arm^o(t)=\arm_{i})\bigg) \notag \\
     = & O\bigg(\frac{n}{\varepsilon^2} \log\big(\frac{k}{\delta}\big)\bigg),
     \end{align}
    where we apply~\eqref{eq:sample-topk} in second inequality. 
\end{proof}

\section{Proof of Theorem \ref{thm:id-BAI}}
\label{appendix-3}

Let $\arm^*$ be the best arm. For $B_r>0$, let $\hat{p}_{*}^{\ell}(r)$ be the estimated mean of $\arm^*$ at $\ell$-th iteration of round $r$.   For $\arm_i\neq \arm^*$, let $r_i=\log_2 (1/(\mu_{\arm^*}-\mu_{\arm_i}))$. Let $\mathcal{E}$ be the event that the best arm is kept in $S_r$ for all $r$, i.e., the returned arm is the best arm.  
      \begin{lemma} [Correctness]
      \label{lem:ins-correc}
      With probability at least $1-\delta/2$, the  returned arm is $\arm^*$.
      \end{lemma}                      
\begin{proof}
Assume that at the $r$-th round, $\arm^*\in S_{r}$. If $\arm^o_r$ is the best arm, then from the algorithm, $\arm^o_r$ is kept at $S_r$. Therefore, we focus on the case that $\arm^o_r$ is not the best arm. From Hoeffding bound,  we have
\begin{align}
\label{eq:bound-Ir}
    \Pr(I_r\leq \mu_{\arm^o_r}+\frac{\varepsilon_r}{2}) \geq 1-\frac{\delta_r}{2}.
    \end{align}
For arriving $\arm^*$, we consider two cases. 

Case 1: $B_r>0$. From Hoeffding bound, at the $\ell$-th iteration, 
\begin{align}
\label{eq:bound-pir-1}
    \Pr(|\hat{p}_*^{\ell}(r)-\mu_{\arm^*}|\geq \frac{\varepsilon_r}{2})\leq \bigg(\frac{\delta_r}{20}\bigg)^{2^{\ell-1}}\leq \frac{\delta_r}{20^{\ell}}.
\end{align}
Applying union bound for all iterations $\ell$, we obtain for all $\ell$,
\begin{align*}
     \Pr(|\hat{p}_*^{\ell}(r)-\mu_{\arm^*}|\geq \frac{\varepsilon_r}{2})\leq \frac{\delta_r}{2}.
\end{align*}

Case 2: $B_r<0$. Let $\hat{p}_{*}(r)$ be the estimated mean of $\arm^*$ for this case. From Hoeffding bound, for  arriving $\arm^*$,
\begin{align}
\label{eq:bound-pir-2}
    \Pr(|\hat{p}_*(r)-\mu_{\arm^*}|\geq \frac{\varepsilon_r}{2})\leq \frac{\delta_r}{20}\leq \frac{\delta_r}{2}.
\end{align}
%Applying union bound, for every $ \arm_i \in S_r\setminus \{\arm^o_r\}$, 
%\begin{align}
%\label{eq:bound-pir}
% \Pr(|\hat{p}_i(r)-\mu_{\arm_i}|\geq \frac{\varepsilon}{2})\leq \sum_{h=1}^{\infty}\frac{\delta_r}{20\cdot h^2} \leq \frac{\delta_r}{2}.
%\end{align}
Combining \eqref{eq:bound-Ir}, \eqref{eq:bound-pir-1}, and \eqref{eq:bound-pir-2} together, we have that with probability at least $1-\delta_r$,
\begin{align}
    \hat{p}_*(r)\geq \mu_{\arm^*}-\frac{\varepsilon_r}{2}\geq \mu_{\arm_r^o}-\frac{\varepsilon_r}{2}\geq I_r-\varepsilon_r, \qquad \text{for}  \ \ B_r\leq 0 \\
    \max_{\ell}\hat{p}^{\ell}_*(r)\geq \mu_{\arm^*}-\frac{\varepsilon_r}{2}\geq \mu_{\arm_r^o}-\frac{\varepsilon_r}{2}\geq I_r-\varepsilon_r, \qquad \text{for} \ \ B_r>0.
\end{align}
Hence, with probability at least $1-\delta_r$, $\arm^*$ will be kept in $S_{r+1}$.
Applying union bound,  $\arm^*$  is kept in $S_r$ for all $r$ rounds with probability at least $1-\sum_{r=1}^{\infty}\delta_r=1-\sum_{r=1}^{\infty}(\delta/(40\cdot r^2))\geq 1-\delta/2$. Therefore with probability at least $1-\delta/2$, the returned arm is $\arm^*$.
\end{proof}      
      
\begin{lemma}
\label{lem-ins-1}
 For $\arm_i\neq \arm^*$, let $r_i=\lceil \log_2 (1/(\mu_{\arm^*}-\mu_{\arm_i}))\rceil $,  then  $\Pr(\arm_i\in S_{r_i+t} \mid \mathcal{E})\leq 1/10^t$.
\end{lemma}
\begin{proof}
Assume that at the $(r_i+t)$-th round,  $\arm_i$ is kept in $S_r$.
Conditioned on event  $\mathcal{E}$, from Theorem~\ref{alg:main}, with probability $1-\delta_{r_i+t}$, 
\begin{align}
\label{eq:ins-sample-1}
 \mu_{\arm^o_{r_i+t}}\geq \mu_{\arm^*}-{\eps_{r_i+t}} \geq \mu_{\arm^*}-{\eps_{r_i}}=\mu_{\arm^*}- \frac{\mu_{\arm^*}-\mu_{\arm_i}}{4}.
\end{align}
From Hoeffding bound,  with probability $1-\delta_{r_i+t}$
\begin{align}
\label{eq:ins-sample-2}
    I_r\geq \mu_{\arm^o_{r_i+t}}-{\eps_{r_i+t}}\geq  \mu_{\arm^o_{r_i+t}}- \frac{\mu_{\arm^*}-\mu_{\arm_i}}{4}.
\end{align}
Again for $\arm_i$, from Hoeffding bound and union bound, we have that conditioned on $\mathcal{E}$, with probability $1-\delta_{r_i+t}$,
\begin{align}
\label{eq:ins-sample-3}
  &  \hat{p}_i(r_i+t)\leq \mu_{\arm_i}+\eps_{r_i+t} \leq \mu_{\arm_i}+\eps_{r_i} \leq \mu_{\arm_i}+ \frac{\mu_{\arm^*}-\mu_{\arm_i}}{4}, \qquad \text{for} \  \ B_r\leq 0; \notag \\
   & \max_{\ell}\hat{p}_i^{\ell}(r_i+t) \leq \mu_{\arm_i}+ \frac{\mu_{\arm^*}-\mu_{\arm_i}}{4}, \qquad \text{for} \  \ B_r>0.
\end{align}
Combining \eqref{eq:ins-sample-1}, \eqref{eq:ins-sample-2}, and \eqref{eq:ins-sample-3} together, we have that with probability $1-3\delta_{r_i+t}\geq 1-\delta/10$,
\begin{align}
    \hat{p}_i(r_i+t)\leq \mu_{\arm_i}+ \frac{\mu_{\arm^*}-\mu_{\arm_i}}{4} &=\mu_{\arm^*}- \frac{3(\mu_{\arm^*}+\mu_{\arm_i})}{4}\leq I_r-\frac{\mu_{\arm^*}+\mu_{\arm_i}}{4}\leq I_r-\eps_{r_i+t},  \qquad \text{for} \  \ B_r<0; \label{eq:in-cam} \\
   & \max_{\ell}\hat{p}_i^{\ell}(r_i+t)\leq  I_r-\eps_{r_i+t}, \ \ \text{for} \  \ B_r>0 \notag
\end{align}
where the second inequality in \eqref{eq:in-cam} is due to 
\begin{align*}
    I_r \geq \mu_{\arm^o_{r_i+t}}- \frac{\mu_{\arm^*}-\mu_{\arm_i}}{4}\geq \mu_{\arm^*}-\frac{\mu_{\arm^*}-\mu_{\arm_i}}{2}.
\end{align*}
Hence,  $\Pr(\arm_i\in S_{r_i+t+1}\mid \arm_i \in S_{r_i+t},\mathcal{E} )\leq \delta/10$. We finally have 
\begin{align*}
    \Pr(\arm_i\in S_{r_i+t} \mid\mathcal{E})&\leq     \Pr(\arm_i\in S_{r_i+t}\mid\mathcal{E})/     \Pr(\arm_i\in S_{r_i}\mid\mathcal{E}) \\ \notag
    &=\prod_{s=0}^{t-1} \Pr(\arm_i\in S_{r_i+s+1}\mid \arm_i \in S_{r_i+s},\mathcal{E}) \\
    &
\leq \delta/10^t,
\end{align*}
where the second equality we use the fact that $\Pr(A\mid B,C)=\Pr(A,B\mid C)/\Pr(B\mid C)$ and $\Pr(\arm_i\in S_{r_i+s+1}, \arm_i \in S_{r_i+s}\mid\mathcal{E} )=\Pr(\arm_i\in S_{r_i+s+1}\mid\mathcal{E} )$.
\end{proof}

 \begin{lemma} 
   \label{lem:ins-comp}
Conditioned on event $\mathcal{E}$, the expected number of pulls of Algorithm~\ref{alg:instance-dependent} is $O\big(\sum_{i=2}^{n}\frac{1}{\Delta_i^2} \log \big(\frac{1}{\delta}\log\frac{1}{\Delta_i} \big) \big)$.
 \end{lemma}

\begin{proof}
    The proof of sample complexity follows the similar idea as \citet{karnin2013almost}. Note that in the $r$-th round, each arm costs $O(\frac{1}{\eps_r^2}\log \frac{1}{\delta_r})$ samples in expectation. Without loss of generality, we assume $\arm_i$ is pulled $O\big(\frac{1}{\eps_r^2}\log \frac{1}{\delta_r}\big)$ times in the $r$-th round. We consider $r$ in two cases, \ie $r\in[1,r_i)$ and $r\geq r_i$. For $r\in [1,r_i)$, the  total sample cost of $\arm_i$ is $O(\frac{1}{\eps_{r_i}^2})\log\big(\frac{1}{\delta_{r_i}}\big)=O\big(\frac{1}{\Delta_i^2}\log \big({\delta}{\log \Delta_i^{-1}}\big)\big)$.  Note that for $(r_i+t)$-th round, the total number of pulls of $\arm_i$ is $O\big(\frac{1}{\eps_{r+t}^2}\cdot \log\big(\frac{1}{\delta_{r+t}}\big)\big) =O(\frac{4^t}{\eps_{r}^2}\cdot \log\big(\frac{1}{\delta_{r}} \big))$.
    For $r>r_i$, applying Lemma \ref{lem-ins-1}, the sample cost of $\arm_i$ is bounded by
   $O\big(\frac{1}{\Delta_i^2}\log \big({\delta}{\log \Delta_i^{-1}}\big)\big)\cdot \sum_{t=1}^{\infty} \frac{4^t}{10^t}=O\big(\frac{1}{\Delta_i^2}\log \big({\delta}{\log \Delta_i^{-1}}\big)\big)$. Therefore, conditioned on event $\cE$, the expected number of pulls is $O\big(\sum_{i=2}^{n}\frac{1}{\Delta_i^2} \log \big(\frac{1}{\delta}\log\frac{1}{\Delta_i} \big) \big)$.
\end{proof}
 
 \begin{lemma}
  \label{lem:ins-pass}
Conditioned on event $\cE$, the expected number of passes used in Algorithm~\ref{alg:instance-dependent} is $O(\log \Delta_{2}^{-1})$.
 \end{lemma}     
\begin{proof}
We divide the $\varepsilon_r$ into two parts.
The first part is $\varepsilon_r\in [\Delta_2/3, 1]$. In this part,  we need at most $\log \Delta_2^{-1}$ rounds. For each round, we need to run \eps-BAI one time and use $O(1)$ passes.
Thus, for the first part, we need $O(\log \Delta_2^{-1} )$ rounds. 

The second part is $\eps_r < \Delta_2/3$. For $\varepsilon_r < \Delta_2/3$, conditioned on $\cE$,  $\arm^o_r$ is the best arm. Applying Hoeffding bound, with probability at least $1-\delta_r$, $I_r\geq \mu_{\arm_r^o}-\Delta_2/6\geq \mu_{\arm^*}-\Delta_2/2$ holds. Let $\cE_1$ be the event that $B_r>0$  holds through the $r$-th round. We first assume $\cE_1$ holds. 
For $\arm_i$ in the stream, applying Hoeffding bound, $\hat{\mu}_{\arm_i}\leq \mu_{\arm_i}+\Delta_i/6\leq \mu_{\arm^*}-5\Delta_2/6\leq I_r-\varepsilon_r$ holds with probability at least $1-\delta_r/(40h^2)$ and thus $\arm_i$ will be removed.  Applying union bound, total probability is at least $1-\sum_{h=1}\frac{\delta_r}{40h^2}\geq 1-\delta_r$.  Next, we show that $\cE_1$ holds with high probability.  Note that $B_r=\frac{6|S_r|}{\varepsilon_r^2}\log\big(\frac{40}{\delta_r} \big)$. 
From \eqref{eq:bound-pir-1}, we know that with probability at least $1-{\delta_r}/{20^\ell}$, $\arm_i$ will be eliminated at the $\ell$-th iteration. Therefore, the expected number of pulls of $\arm_i$  is bounded by 
\begin{align*}
    X_i\leq M+\sum_{\ell=1}^{\infty}\frac{2^{\ell}M}{20
    ^{\ell}}\leq \frac{6M}{5},
\end{align*}
where $M=\frac{2}{\varepsilon^2}\log\big(\frac{40}{\delta_r}\big)$. From Markov inequality,
\begin{align*}
    \Pr(\cE_1)\geq\Pr\bigg(\sum_{\arm_i\in S_r}X_i\leq B_r\bigg)=1-\Pr\bigg(\sum_{\arm_i\in S_r}X_i\geq B_r\bigg) \geq 1- \frac{\EE\bigg[\sum_{\arm_i\in S_r}X_i\bigg]}{B_r}\geq \frac{4}{5}.
\end{align*}
Therefore,  conditioned on $\cE$, there are two cases. If $\arm^o_r$ is an optimal arm. Then all suboptimal arms  will be eliminated at the round $r$ with probability at least $1-\delta_r-1/5$.  Note that $1-\delta_r-1/5\geq 3/5$, the number of passes for second part is $O(1)$. If $\arm^o_r$ is a suboptimal arm, all suboptimal arms  will be eliminated at round $r$ with probability at least $1-\delta_r-1/5$ except for $\arm^o_r$. Obviously, conditioned on $\cE$, $\arm^o_r$ will be eliminated in the next round with  high probability. Thus, the number of passes is $O(1)$.
\end{proof}

Combining Lemma \ref{lem:ins-correc},  \ref{lem:ins-comp} and \ref{lem:ins-pass}, we get Theorem \ref{thm:id-BAI}.

%\section{Asymptotic BAI}

% We first introduce the asymptotic lower bound of sample complexity for Bernoulli reward distribution. Note that the asymptotic lower bound not only holds for Bernoulli distributions but more generally for one-parameter exponential families~\cite{garivier2016optimal}. Denote $\cS$ for a set of bandit instances in which there exists a unique optimal arm. \cite{garivier2016optimal} proves the following asymptotic lower bound. 
%\begin{proposition} [\cite{garivier2016optimal}]
%Let $\delta\in(0,1)$,  for any strategy and any Bernoulli bandit instance in $\cS$, 
%\begin{align*}
%    \lim \inf_{\delta \rightarrow 0} \frac{\EE_{\mu}[N_{\delta}]}{\log (1/\delta)}\geq T^*(\mu),
%\end{align*}
%where the $N_\delta$ is the number of samples and $T^*(\mu)$ is the maximum of the optimization problem
%\begin{align*}
%  \max_{w\in \sum_{n}} \inf_{\arm_i \neq \arm^*} \bigg[w^* \cdot d\bigg(\mu_{\arm^*},\frac{w^*\mu_{\arm^*}+w_i \mu_{\arm_i}}{w^*+w_i}\bigg)+w_i \cdot d\bigg(\mu_{\arm_i}, \frac{w^*\mu_{\arm_i}+w_i \mu_{\arm^*}}{w^*+w_i} \bigg) \bigg],
%\end{align*}
%where $\sum_{n}$ is the set of probability distribution of $\{1,2,\cdots, n\}$ and $d(x,y)$ is the relative entropy between Bernoulli distributions with parameters $x,y$, i.e., $d(x,y)=x\log (x/y)+(1-x)\log ((1-x)/(1-y))$.
%\end{proposition}    

\end{document}